\documentclass[a4paper]{article}
\usepackage{fullpage}
\usepackage[utf8]{inputenc}
\usepackage{latexsym,amsfonts,amssymb,amsthm,amsmath}
\usepackage{color}
\usepackage{cite}
\usepackage{thm-restate}
\usepackage{tikz}
\usetikzlibrary{calc}
\usetikzlibrary{decorations.pathreplacing}

\theoremstyle{plain}
\newtheorem{theorem}{Theorem}[section]
\newtheorem{lemma}[theorem]{Lemma}

\theoremstyle{definition}
\newtheorem{definition}[theorem]{Definition}
\newtheorem{remark}[theorem]{Remark}

\newcommand{\comment}[1]{}

\newtheorem{proposition}[theorem]{Proposition}

\newcommand{\N}{\mathbb{N}}
\newcommand{\Z}{\mathbb{Z}}
\newcommand{\R}{\mathbb{R}}
\newcommand{\mb}[1]{\ensuremath{\boldsymbol{#1}}}
\usepackage{array}

\allowdisplaybreaks
\usepackage{authblk}
\title{The Representation Power of Neural Networks: \\Breaking the Curse of Dimensionality}
\author{Mo\"ise Blanchard and Amine Bennouna \\
\texttt{moiseb@mit.edu} \; \; \texttt{amineben@mit.edu}
}

\affil{Operations Research Center, Massachusetts Institute of Technology, Cambridge MA}

\date{}

\begin{document}

\maketitle

\begin{abstract}
In this paper, we analyze the number of neurons and training parameters that a neural network needs to approximate multivariate functions of bounded second mixed derivatives --- Korobov functions. We prove upper bounds on these quantities for shallow and deep neural networks, breaking the curse of dimensionality. Our bounds hold for general activation functions, including ReLU. We further prove that these bounds nearly match the minimal number of parameters any continuous function approximator needs to approximate Korobov functions, showing that neural networks are near-optimal function approximators.
\end{abstract}

\section{Introduction}
Neural networks have known tremendous success in the past years in many applications such as computer vision and pattern detection \cite{krizhevsky2017imagenet, silver2016mastering}. A natural question that arises is how to explain their practical success theoretically.

Neural networks are shown to be universal \cite{hornik1989multilayer, cybenko1989approximation}. Any measurable function on a finite dimensional space can be approximated arbitrarily well by a neural network provided the network has a sufficient number of neurons. Furthermore, universality holds for as low as one hider layer neural network with reasonable activation functions. However, these results do not specify the number of neurons and parameters to train needed for this approximation. If such numbers are unreasonably high, the universality of neural networks would not explain their practical success.

In this paper, we are interested in evaluating the number of neurons and training parameters needed to approximate a given function within $\epsilon$ with a neural network. An interesting question is how do these numbers scale with $\epsilon$ and the dimensionality of the problem, that is, the number of variables of the function. 
Mhaskar \cite{mhaskar1996neural} shows that any function of the Sobolev space of order $r$ and dimension $d$ can be approximated within $\epsilon$ with a one layer neural network with $O(\epsilon^{-\frac{d}{r}})$ neurons and an infinitely differentiable activation function. This bound exhibits the curse of dimensionality: the number of neurons needed for an approximation within $\epsilon$ scales exponentially in the dimension of the problem $d$. Therefore, Mhaskar's bound raises the important question of whether this curse is inherent to neural networks.

DeVore and al. \cite{devore1989optimal} showed an important result towards answering this question. They proved that any continuous function approximator (a notion we formalize in Section \ref{sec: NN are optimal.}) that approximates all Sobolev functions of order $r$ and dimension $d$ within $\epsilon$, needs at least $\Theta(\epsilon^{-\frac{d}{r}})$ parameters. This result bridges the gap with Mhaskar's bound and confirms that neural networks cannot escape the curse of dimensionality for the Sobolev space. An important question is therefore for which set of functions can neural networks break this curse of dimensionality.

One way to circumvent the curse of dimensionality is by restricting considerably the considered space of functions and focusing on functions with particular structures. For example Mhaskar et al. \cite{mhaskar2016learning} showed that compositional functions with regularity $r$ can be approximated within $\epsilon$ with deep neural networks with $O(d\cdot \epsilon^{-\frac{2}{r}})$ neurons. Another example is the class of Sobolev functions of order $d/\alpha$ and dimension $d$ for which Mhaskar's bound becomes $O(\epsilon^{-\alpha})$.
% However, these considered classes of functions are very restrictive and lack generality.
Recently, Montanelli et al. \cite{montanelli2019deep} considered the case of bandlimited functions and showed that they can be approximated within $\epsilon$ by deep networks with depth $O((\log\frac{1}{\epsilon})^2)$ and $O(\epsilon^{-2} (\log\frac{1}{\epsilon})^2)$ neurons. Weinan et al. \cite{weinan2019barron} showed that the closure of the space of 2-layer neural networks with a particular regularity (namely a restriction on the size of the network's weighs) is the Barron space. They further show that Barron functions can be approximated within $\epsilon$ with 2-layer networks with $O(\epsilon^{-2})$ neurons. 

In this work, we are interested in more general and generic spaces of functions. Our space of interest is the space of functions of bounded second mixed derivatives, namely the Korobov space. This space is included in the Sobolev space but is reasonably large and general. Recently, Montanelli and Du \cite{montanelli2019new} showed that deep neural networks with depth $O(\log \frac{1}{\epsilon})$ and $O(\epsilon^{-\frac{1}{2}} (\log\frac{1}{\epsilon})^{\frac{3(d-1)}{2}+1})$ neurons can approximate Korobov functions within $\epsilon$, breaking the curse of dimensionality for deep neural network.\comment{ Their bound is the first to break the curse of dimensionality for the Korobov space.} While they used deep structures to prove their result, the question of whether shallow neural networks also break the curse of dimensionality for the Korobov space remains open. 

% DeVor's lower bound on the number of neurons needed depends heavily on the considered function space -- Sobolev in the case of Mhaskar. One way to circumvent the curse of dimensionality is by restricting the considered space of functions. By imposing heavy, and rather unrealistic regularity conditions of the considered functions, such as $r \approx \alpha d$, we obtain a tight bound of $O(\epsilon^{-\frac{1}{\alpha}})$ on the number of needed neurons. However, one might not necessarily need such heavy conditions. By taking a slightly smaller functional space than the Sobolev space of order 2, namely the space of functions of bounded second derivatives, DeVor's lower bound becomes $\Theta (\frac{1}{\sqrt{\epsilon}})$ and the curse of dimensionality vanishes. This result opens up new opportunities in considering whether neural networks as functional approximators suffer from the curse of dimensionality.

In this paper, we study deep and shallow neural network's approximation power for the Korobov space and make the following contributions:
\begin{itemize}
\item
\textbf{Representation power of shallow neural networks.} We prove that any function of the Korobov space can be approximated within $\epsilon$ with a 2-layer neural network with ReLU activation function with $O(\epsilon^{-1}(\log\frac{1}{\epsilon})^{(3d-1)/2})$ neurons and $O(\epsilon^{-\frac{1}{2}}(\log\frac{1}{\epsilon})^{(3d-1)/2})$ training parameters (Theorem \ref{thm: shallow NN bound.}). We further extend this result to a large class of commonly used activation functions (Theorem \ref{thm: shallow NN bound generalized activation.}). Our new bound can be written as $O(\epsilon^{-1-\delta})$ for all $\delta>0$, and therefore breaks the curse of dimensionality for shallow neural networks.
\item 
\textbf{Representation power of deep neural networks.} We prove that any function of the Korobov space can be approximated within $\epsilon$ with a deep neural network of depth $\lceil \log_2(d) \rceil +1$ independent of $\epsilon$, with non-linear $\mathcal C^2$ activation function, $O(\epsilon^{-\frac{1}{2}}(\log\frac{1}{\epsilon})^{(3d-1)/2})$ neurons and $O(\epsilon^{-\frac{1}{2}}(\log\frac{1}{\epsilon})^{(3d-1)/2})$ training parameters (Theorem \ref{thm: fixed deep upper bound}). This result improves over the result of Montanelli and Du \cite{montanelli2019new} who constructed an approximating neural network with larger depth $O(\log \frac{1}{\epsilon}\log d)$ --- increasing with the precision $\epsilon$ --- and larger number of neurons $O(\epsilon^{-\frac{1}{2}}(\log\frac{1}{\epsilon})^{(3d-1)/2+1})$. We note however that they used a different activation function, namely ReLU.
\item \textbf{Near-optimality of neural networks as function approximators.} Under DeVore's \cite{devore1989optimal} model of continuous function approximators, we prove that any continuous function approximator needs at least $\Theta(\epsilon^{-\frac{1}{2}}( \log\frac{1}{\epsilon})^{\frac{d-1}{2}})$ parameters to approximate functions of the Korobov space within $\epsilon$ (Theorem \ref{thm:lower bound korobov}). This lower bound nearly matches our established upper bounds on the number of training parameters needed by deep and shallow neural networks to approximate functions of the Korobov space, proving that deep and shallow neural networks are near-optimal function approximators of the Korobov space. We further show that neural networks are optimal function approximators for separable functions.
\end{itemize}

Table \ref{tab:shallow} and \ref{tab:deep} summarize our new bounds and the existing bounds for shallow and deep neural network approximation power for the Korobov space, Sobolev space and bandlimited functions. Our proofs are constructive, and present explicit structures to construct such deep and shallow neural networks with ReLU and general activation functions. Our constructions relie on sparse grid approximations introduced by Zenger \cite{zenger1991sparse}, and studied in detail by Bungartz and Griebel  \cite{bungartz1992dunne,bungartz2004sparse}.

The rest of the paper is structured as follows. In Section \ref{sec: preliminaries}, we define the notions of interest, formalize our objective and introduce the sparse grids approach which will be key in our later constructions. In Section \ref{sec: shallow nn.} and \ref{sec: deep nn.}, we prove our bounds on the number of neurons and training parameters needed by a neural network to approximate functions of the Korobov space. Section \ref{sec: shallow nn.} provides bounds for shallow networks and Section \ref{sec: deep nn.} for deep networks. Finally, we formalize in Section \ref{sec: NN are optimal.} the notion of optimal continuous function approximators and prove our near-optimality result.

\begin{table}[h]
\begin{center}
\begin{tabular}{ |c|c|c|c|c|c| } 
 \hline
 Space & \begin{tabular}{c} Number of\\ neurons\end{tabular}& \begin{tabular}{c} Number of\\ training\\
 parameters\end{tabular} & Depth & Activation $\sigma$ & Ref. \\
 \hline
 \hline
 \begin{tabular}{c}Sobolev\\$W^{r,p}(\Omega)$\end{tabular} & $ \epsilon^{-\frac{d}{r}}$ & $ \epsilon^{-\frac{d}{r}}$ & 1 &\begin{tabular}{c} $\mathcal C ^\infty$,\\ non-polynomial\end{tabular} & \cite{mhaskar1996neural}\\
 \hline
 \begin{tabular}{c}Korobov\\$X^{2,\infty}(\Omega)$\end{tabular} & $ \epsilon^{-1}(\log \frac{1}{\epsilon})^{\frac{3(d-1)}{2}+1}$ & $ \epsilon^{-\frac{1}{2}}(\log \frac{1}{\epsilon})^{\frac{3(d-1)}{2}}$ & 2 & \begin{tabular}{c} any ReLU-like, e.g.\\ ReLU, ELU, soft-plus\end{tabular} & \begin{tabular}{c} \textbf{This}\\  \textbf{paper}\end{tabular}\\
 \hline
 \begin{tabular}{c}Korobov\\$X^{2,\infty}(\Omega)$\end{tabular} & $ \epsilon^{-\frac{3}{2}}(\log \frac{1}{\epsilon})^{\frac{3(d-1)}{2}}$ & $ \epsilon^{-\frac{1}{2}}(\log \frac{1}{\epsilon})^{\frac{3(d-1)}{2}}$ & 2 & \begin{tabular}{c} any Sigmoid-like, e.g.\\ sigmoid, step, tanh\end{tabular} & \begin{tabular}{c} \textbf{This}\\  \textbf{paper}\end{tabular}\\
 \hline
\end{tabular}
\end{center}
\caption{Approximation results for Sobolev and Korobov functions by shallow neural networks. Number of neurons and training parameters are given in $O$ notation.}
    \label{tab:shallow}
\end{table}

\begin{table}[h]
\begin{center}
\begin{tabular}{ |c|c|c|c|c|c| } 
 \hline
 Space & \begin{tabular}{c} Number of\\ neurons\end{tabular}& \begin{tabular}{c} Number of\\ training\\
 parameters\end{tabular} & Depth & Activation $\sigma$ & Ref. \\
 \hline
 \hline
 \begin{tabular}{c}Sobolev\\$W^{r,p}(\Omega)$\end{tabular} & $ \epsilon^{-d/r} \log \frac{1}{\epsilon}$ & $ \epsilon^{-d/r} \log \frac{1}{\epsilon}$ & $O( \log \frac{1}{\epsilon})$ & ReLU & \cite{yarotsky2017error, liang2016deep}\\
 \hline
 \begin{tabular}{c}Bandlimited\\functions\end{tabular} & $\epsilon^{-2} (\log\frac{1}{\epsilon})^2$ & $\epsilon^{-2} (\log\frac{1}{\epsilon})^2$ & $O((\log\frac{1}{\epsilon})^2)$& ReLU & \cite{montanelli2019deep}\\
 \hline
 \begin{tabular}{c}Korobov\\$X^{2,\infty}(\Omega)$\end{tabular} & $ \epsilon^{-\frac{1}{2}}(\log \frac{1}{\epsilon})^{\frac{3(d-1)}{2}+1}$ & $ \epsilon^{-\frac{1}{2}}(\log \frac{1}{\epsilon})^{\frac{3(d-1)}{2}}$ & $ O(\log d \cdot \log \frac{1}{\epsilon})$ & ReLU & \cite{montanelli2019new}\\
 \hline
 \begin{tabular}{c}Korobov\\$X^{2,\infty}(\Omega)$\end{tabular} & $ \epsilon^{-\frac{1}{2}}(\log \frac{1}{\epsilon})^{\frac{3(d-1)}{2}}$ & $ \epsilon^{-\frac{1}{2}}(\log \frac{1}{\epsilon})^{\frac{3(d-1)}{2}}$ & $\lceil \log_2 d \rceil +1$ & \begin{tabular}{c} $\mathcal C^2$,\\non-linear\end{tabular} & \begin{tabular}{c} \textbf{This}\\  \textbf{paper}\end{tabular}\\
 \hline
\end{tabular}
\end{center}
    \caption{Approximation results for Sobolev and Korobov functions by deep neural networks. Number of neurons and training parameters are given in $O$ notation.}
    \label{tab:deep}
\end{table}

\section{Preliminaries}\label{sec: preliminaries}

We begin by presenting the basic definitions that will be used in the paper. In this work, we consider feed-forward neural networks, using a linear output neuron and a non-linear activation function $\sigma:\R\to\R$ for the other neurons, such as the commonly-used rectified unit (ReLU) $\sigma(x) = \max(x,0)$, the sigmoid $\sigma(x) = (1+e^{-x})^{-1}$ or the Heaviside function $\sigma(x) = \mb{1}_{\{x\geq 0\}}$. Let $d\geq 1$ be the dimension of the input. We define a $1-$hidden layer network with $N$ neurons as
\begin{equation*}
    \mb{x}\mapsto \sum_{k=1}^N u_k \sigma(\mb{w_k}^\top \mb{x} + b_k),
\end{equation*}
where $\mb{w_k}\in \R^d$, $b_k\in \R$ for $i=1,\cdots,N$ are the parameters of the network. A neural network with several hidden layers is obtained by feeding the outputs of the neurons in the a given layer as inputs to the next layer. For example, a $2-$hidden layer network with $N_1$ (resp. $N_2$) neurons on the first (resp. second) hidden layer is defined as
\begin{equation*}
    \mb{x}\mapsto \sum_{i=1}^{N_2} u_i  \sigma\left(\sum_{j=1}^{N_1} v_{i,j} \sigma(\mb{w_j}^\top \mb{x} +c_j)+ b_i\right).
\end{equation*}
We are interested in the expressive power of neural networks i.e. the ability to approximate a target function $f:\R^d\to\R$ with as few neurons as possible, on the unit hyper-cube $\Omega :=[0,1]^d$. Another relevant metric is the number of parameters that needs to be trained to approximate the function, that is, the number of parameters of the approximating network ($\mb{w}_j$, $c_j$, $v_{i,j}$ and $b_i$) depending on the function to approximate. We will adopt $L^\infty$ norm as a measure of approximation error.

Let us now define some notations that are necessary to present the space of functions that we will approximate. For an integer $r$, we denote $\mathcal{C}^r$ the space of one dimensional functions differentiable $r$ times and with continuous derivatives. In our analysis, we consider functions $f$ with bounded mixed derivatives. For a multi-index $\mb{\alpha}\in \N^d$, we define the derivative of order $\alpha$
\begin{equation*}
    D^{\mb{\alpha}} f := \frac{\partial^{|\mb \alpha|_1}f}{\partial x_1^{\alpha_1}\cdots\partial x_d^{\alpha_d}}.
\end{equation*}
Two common function spaces in a compact $\Omega\subset \R^d$ are the Sobolev spaces $W^{r,p}(\Omega)$ of functions having weak partial derivatives up to order $r$ in $L^p(\Omega)$ and the Korobov spaces $X^{r,p}(\Omega)$ of functions vanishing at the boundary and having weak mixed second derivatives up to order $r$ in $L^p(\Omega)$. More precisely,
\begin{align*}
    W^{r,p}(\Omega)&=\{f\in L^p(\Omega)\; : \; D^{\mb \alpha}f \in L^p(\Omega), |\mb \alpha|_1\leq r\},\\
    X^{r,p}(\Omega)&=\{f\in L^p(\Omega)\; : \;f|_{\partial\Omega}=0, D^{\mb \alpha}f \in L^p(\Omega), |\mb \alpha|_\infty\leq r\}.
\end{align*}
where $\partial \Omega$ denotes the boundary of $\Omega$, $|\mb\alpha|_1 = \sum_{i=1}^d |\alpha_i|$ and $|\mb\alpha|_{\infty} = \sup_{i=1,\ldots,d} |\alpha_i|$ are respectively the $L^1$ and infinity norm. The usual norms on these spaces are given by
\begin{equation*}
    |f|_{W^{r,p}(\Omega)} := \begin{cases} 
    \left(\sum_{|\mb \alpha|_1\leq r} \left\|D^{\mb\alpha}f\right\|_p^p\right)^{\frac{1}{p}} & 1\leq p<\infty, \\
     \max_{|\mb \alpha|_1\leq r} \left\|D^{\mb\alpha}f\right\|_\infty     & p=\infty,
       \end{cases}
       \quad
       |f|_{X^{r,p}(\Omega)} := \begin{cases} 
    \left(\sum_{|\mb \alpha|_\infty\leq r} \left\|D^{\mb\alpha}f\right\|_p^p\right)^{\frac{1}{p}} & 1\leq p<\infty, \\
     \max_{|\mb \alpha|_\infty\leq r} \left\|D^{\mb\alpha}f\right\|_\infty     & p=\infty.
       \end{cases}
\end{equation*}
Note that Korobov spaces $X^{r,p}(\Omega)$ are subsets of Sobolev spaces $W^{r,p}(\Omega)$. We will focus our analysis on approximating functions on the Korobov space $X^{2,\infty}(\Omega)$ for which we will show that neural networks are near-optimal.

\subsection{Breaking the curse of dimensionality}
\label{section:break curse}

Previous efforts to quantify the number of neurons needed to approximate a class of functions within $\epsilon$ showed that neural networks and most classical functional approximation schemes exhibit the curse of dimensionality. For example, in the case of the case of Sobolev functions, Mhaskar proved the following approximation bound.
\begin{theorem}[Mhaskar \cite{mhaskar1996neural}]\label{thm: Mhaskar upper bound}
Let $p,r\geq 1$, and $\sigma:\R\to\R$ be an infinitely differentiable activation function, non-polynomial on any interval of $\R$. Let $\epsilon>0$ sufficiently small. For any $f\in W^{r,p}$, there exists a shallow neural network with one hidden layer, activation function $\sigma$, and $O\left(\epsilon^{-\frac{d}{r}}\right)$ neurons approximating $f$ within $\epsilon$ for the infinity norm.
\end{theorem}

The approximation of Sobolev functions by neural networks therefore suffers from the curse of dimensionality since the number of neurons needed grows exponentially with the input space dimension $d$. This curse is not due to poor performance of neural networks but rather to the choice of the Sobolev space to approximate. DeVore and al. \cite{devore1989optimal} proved a general result showing that any learning algorithm with continuous parameters needs at least $\Theta(\epsilon^{-\frac{d}{r}})$ parameters to approximate the Sobolev space $W^{r,p}$. This shows that the class of Sobolev functions suffers inherently from the curse of dimensionality and no function approximators can overcome it. We detail this notion later in Section \ref{sec: NN are optimal.}.

The natural question that arises is whether there exists a reasonable and sufficiently large class of functions for which there is no inherent curse of dimensionality. We are interested in this case in the performance of neural networks in approximating this space of functions. Instead of the Sobolev space, we turn to Korobov space $X^{2,\infty}(\Omega)$ of functions with bounded mixed derivatives. Notice that this space of functions is included in the Sobolev space, but still reasonably large as the regularity condition concerns only second order derivatives. Two questions will be of interest. First, how many neurons and training parameters a neural network needs to approximate any Korobov function within $\epsilon$ in the $L^{\infty}$ norm. Second, how does neural network performance compare to the theoretically best possible function approximators for Korobov space.

\subsection{Sparse grids and hierarchical basis}
\label{subsection: sparse grids}
We review in this subsection the notion of sparse grids which will be key in our neural networks constructions to approximate functions in Korobov spaces. Sparse grids were introduced by Zenger \cite{zenger1991sparse} in order to solve partial differential equations. Sparse grids have since then been extensively used for high-dimensional function approximation. They have been proven to break the curse of dimensionality for Korobov spaces. We refer to Bungartz and Griebel \cite{bungartz2004sparse} for a thorough review of the topic.

The goal is to define discrete approximation spaces with basis functions. Instead of a classical uniform grid partition of the hyper-cube $[0,1]^d$ involving $n^d$ components, where $n$ is the number of partitions in each coordinate, the sparse grid approach uses a smarter partitioning of the cube preserving the approximation accuracy while drastically reducing the number of components of the grid. The construction involves a $1-$dimensional mother function $\phi$ which is used to generate all the functions of the basis. For example, a simple choice the building block $\phi$ is the standard hat function
\begin{equation}\label{eq: hat function}
    \phi(x) := (1-|x|)_+, \quad \forall x \in [-1,1].
\end{equation}
The hat function is not the only possible choice. In the latter proofs we will specify which mother function is used, in our case either the Deslaurier-Dubuc interpolates \cite{deslauriers1989symmetric} (which we define rigorously later in our proofs) or the hat function $\phi$ which can be seen as the Deslaurier-Dubuc interpolate of order $1$. These more elaborate mother functions enjoy more smoothness while essentially preserving same approximation power.

Assume the mother function has support in $[-k,k]$. For $j=1,\cdots,d$, it can be used to generate a set of local functions $\phi_{l_j,i_j}: [0,1] \longrightarrow \mathbb{R}$ for all $l_j\geq 1$ and $1\leq i_j\leq 2^{l_j}-1$, with support on $\left[ \frac{i_j-k}{2^{l_j}},\frac{i_j+k}{2^{l_j}}\right] $ as follows,
\begin{equation}\label{eq: hat functions lj ij}
    \phi_{l_j,i_j}(x) := \phi(2^{l_j}x - i_j ), \quad \forall x \in [0,1].
\end{equation}
We now define a basis of $d-$dimensional functions by taking the tensor product of these $1-$dimensional functions. For all $\mb{l},\mb{i} \in \mathbb{N}^d$ with $\mb l\geq \mb 1$ and $\mb 1\leq \mb i\leq 2^{\mb l}-\mb 1$ where $2^{\mb{l}}$ denotes $(2^{l_1},\ldots, 2^{l_d})$, define
\begin{equation*}
    \phi_{\mb l, \mb i}(\mb{x}) := \prod_{j=1}^d \phi_{l_j,i_j}(x_j), \quad \forall \mb{x} \in [0,1]^d.
\end{equation*}
For a fixed $\mb l \in \mathbb{N}^d$, we will consider the \emph{hierarchical increment space} $W_{\mb l}$ which is the subspace spanned by the functions $\{\phi_{\mb l, \mb i}\; : \; \mb 1\leq \mb i\leq 2^{\mb l}-\mb 1 \}$, as illustrated in Figure \ref{fig:hierarchical basis},
\begin{equation*}
    W_{\mb l}:= \text{span} \{\phi_{\mb l,\mb i}, \; \mb 1\leq \mb i\leq 2^{\mb l}-\mb 1,\; i_j \text{ odd for all }1\leq j\leq d\}.
\end{equation*}

\begin{figure}
    \centering
    \begin{tikzpicture}[scale=0.8]
    \foreach \x in {1}
    \draw (\x*5,0.5) node[below] {$\phi_{1,\x}$};
    \draw [thick] (0,-0.5) -- (5,0.5) -- (10,-0.5);
    \draw[,->] (-0.5,-0.5) -- (10.5,-0.5);
    \draw (0 cm,2pt-0.5cm) -- (0 cm,-2pt-0.5cm) node[anchor=north] {$0$};
    \draw (5 cm,2pt-0.5cm) -- (5 cm,-2pt-0.5cm);
    \draw (10 cm ,2pt-0.5cm) -- (10 cm,-2pt-0.5cm) node[anchor=north] {$1$};
    \draw (-1,0) node[left] {$W_1$};
    
    \foreach \x in {0,1}
    \draw [thick] (\x*5,-2) -- (\x*5 +2.5,-1) -- (\x*5 +5 ,-2);
    \foreach \x in {1,2,3}
    \draw (\x*2.5,2pt-2cm) -- (\x*2.5,-2pt-2cm);
    \draw (10 ,2pt-2cm) -- (10,-2pt-2cm) node[anchor=north] {$1$};
    \draw (0 ,2pt-2cm) -- (0,-2pt-2cm) node[anchor=north] {$0$};
    \draw[,->] (-0.5,-2) -- (10.5,-2);
    \foreach \x in {1,3}
    \draw (\x*2.5,-1) node[below] {$\phi_{2,\x}$};
    
    \foreach \x in {0,1,2,3}
    \draw [thick](\x*2.5,-4) -- (\x*2.5 +1.25,-3) -- (\x*2.5 +2.5 ,-4);
    \foreach \x in {1,2,3,4,5,6,7}
    \draw (\x*1.25 cm,2pt-4cm) -- (\x*1.25,-2pt-4cm);
    \draw (10 ,2pt-4cm) -- (10,-2pt-4cm) node[anchor=north] {$1$};
    \draw (0 ,2pt-4cm) -- (0,-2pt-4cm) node[anchor=north] {$0$};
    \draw[,->] (-0.5,-4) -- (10.5,-4);
    \foreach \x in {1,3,5,7}
    \draw (\x*1.25,-3) node[above] {$\phi_{3,\x}$};
    
    \foreach \x in {0,1,2,3,4,5,6,7}
    \draw [thick](\x*1.25,-6) -- (\x*1.25 +0.625,-5) -- (\x*1.25 +1.25 ,-6);
    \foreach \x in {1,2,3,4,5,6,7,8,9,10,11,12,13,14,15}
    \draw (\x*0.625 cm,2pt-6cm) -- (\x*0.625,-2pt-6cm);
    \draw (10 ,2pt-6cm) -- (10,-2pt-6cm) node[anchor=north] {$1$};
    \draw (0 ,2pt-6cm) -- (0,-2pt-6cm) node[anchor=north] {$0$};
    \draw[,->] (-0.5,-6) -- (10.5,-6);
    \foreach \x in {1,3,5,7,9,11,13,15}
    \draw (\x*0.625,-5) node[above] {$\phi_{4,\x}$};
    
    \foreach \y in {2,3,4}
    \draw (-1,2.5-\y*2) node[left] {$W_{\y}$};
    \end{tikzpicture}
    \caption{Hierarchical basis obtained from the sparse grid construction using the wedge $(1-|\cdot|)_+$ as mother function.}
    \label{fig:hierarchical basis}
\end{figure}
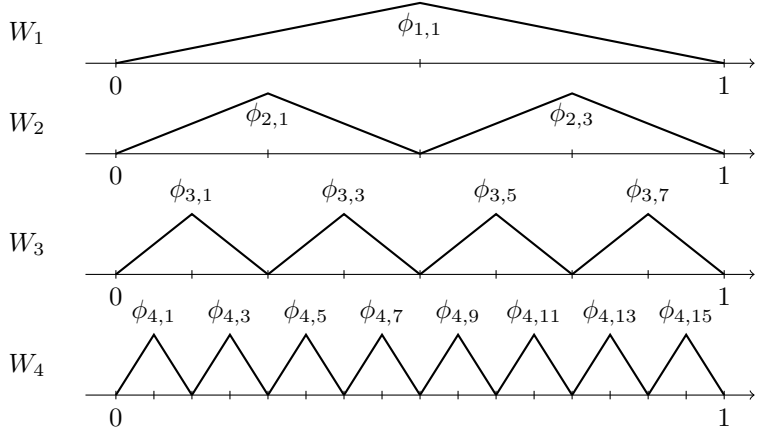
Note that in the hierarchical increment $W_{\mb l}$, all basis functions have disjoint support. A key observation is that Korobov functions $X^{2,p}(\Omega)$ can be expressed uniquely in this hierarchical basis. More precisely, there is a unique representation of $u\in X^{2,p}(\Omega)$ as
\begin{equation*}
u(\mb x) = \sum_{\mb l, \mb i} v_{\mb l,\mb i}\phi_{\mb l, \mb i}(\mb x),
\end{equation*}
where the sum is taken over all multi-indices $\mb l\geq \mb 1$ and $ \mb 1\leq \mb i\leq 2^{\mb l}-\mb 1$ where all components of $\mb i$ are odd. In particular, all basis functions are linearly independent. Notice that this sum is infinite, the objective is now to define a finite-dimensional subspace of $X^{2,p}(\Omega)$ that will serve as approximation space. Sparse grids \cite{bungartz2004sparse} use a carefully chosen subset of the hierarchical basis functions to construct the approximation space $V_n^{(1)}$ as illustrated in dimension $2$ in Figure \ref{fig:sparse grids},
\begin{equation}\label{eq: V1.}
    V_n^{(1)} := \bigoplus_{|\mb l|_1\leq n+d-1} W_{\mb l} .
\end{equation}

\begin{figure}[h!]
    \centering
    \begin{tikzpicture}[scale = 0.9]
    \foreach \x in {0,1,2,3}{
        \foreach \y in {0,1,2,3}
            \draw (2.5*\x,2.5*\y) rectangle (2.5*\x+2,2.5*\y+2);
    }
    \foreach \y in {0,1,2,3}{
        \draw (2.5+1,2.5*\y) -- (2.5+1,2.5*\y+2);
        \draw (2.5*\y,2.5+1) -- (2.5*\y+2,2.5+1);
        \foreach \z in {1,2,3}{
            \draw (2.5*2+\z/2,2.5*\y) -- (2.5*2+\z/2,2.5*\y+2);
            \draw (2.5*\y,2.5*2+\z/2) -- (2.5*\y+2,2.5*2+\z/2);
        }
        \foreach \z in {1,2,3,4,5,6,7}{
            \draw (2.5*3+\z/4,2.5*\y) -- (2.5*3+\z/4,2.5*\y+2);
            \draw (2.5*\y,2.5*3+\z/4) -- (2.5*\y+2,2.5*3+\z/4);
        }
    }
    \foreach \y in {0,1,2,3}{
        \draw (2.5+1,2.5*\y) -- (2.5+1,2.5*\y+2);
        \foreach \z in {1,2,3}
        \draw (2.5*2+\z/2,2.5*\y) -- (2.5*2+\z/2,2.5*\y+2);
        \foreach \z in {1,2,3,4,5,6,7}
        \draw (2.5*3+\z/4,2.5*\y) -- (2.5*3+\z/4,2.5*\y+2);
    }
    \foreach \x in {1,3,4,5.25,5.75,6.25,6.75,7.625,7.875,8.125,8.375,8.625,8.875,9.125,9.375}{
        \foreach \y in {1,3,4,5.25,5.75,6.25,6.75,7.625,7.875,8.125,8.375,8.625,8.875,9.125,9.375}
            \draw (\x,\y) node {$\mb\cdot$};
    }
    \draw [ultra thick,dashed] (9.75,0) -- (9.75,2.25) -- (7.25,2.25) -- (7.25,4.75) -- (4.75,4.75) -- (4.75,7.25) -- (2.25,7.25) -- (2.25,9.75) -- (0,9.75);%respecte moi hahaha c'est pas mal hein! Excellent 
    \draw [thick,->,>=latex] (-0.25,-0.25) -- (10,-0.25) node[right] {$l_1$};
    \draw [thick,->,>=latex] (-0.25,-0.25) -- (-0.25,10) node[above] {$l_2$};
    
    \foreach \y in {0,1,2,3}{
        \draw (-0.5,2.5*\y) -- (-0.5,2.5*\y+2);
        \draw (2.5*\y,-0.5) -- (2.5*\y+2,-0.5);
    }
    \draw (0,-0.5) -- (1,-1) -- (2,-0.5);
    \draw (2.5,-0.5) -- (3,-1) -- (3.5,-0.5) -- (4,-1) -- (4.5,-0.5);
    \draw (-0.5,0) -- (-1,1) -- (-0.5,2);
    \draw (-0.5,2.5) -- (-1,3) -- (-0.5,3.5) -- (-1,4) -- (-0.5,4.5);
    \foreach \x in {5.25,5.75,6.25,6.75}{
        \draw (\x,-1) -- (\x+0.25,-0.5);
        \draw (\x,-1) -- (\x-0.25,-0.5);
        \draw (-1,\x) -- (-0.5,\x+0.25);h
        \draw (-1,\x) -- (-0.5,\x-0.25);
    }
    \foreach \x in {7.625,7.875,8.125,8.375,8.625,8.875,9.125,9.375}{
        \draw (\x,-1) -- (\x+0.125,-0.5);
        \draw (\x,-1) -- (\x-0.125,-0.5);
        \draw (-1,\x) -- (-0.5,\x+0.125);
        \draw (-1,\x) -- (-0.5,\x-0.125);
    }
    \end{tikzpicture}
    \caption{Representation of the hierarchical basis $W_{(l_1,l_2)}$ for $1\leq l_1,l_2\leq 4$, in dimension $2$. The rectangles show the support of the basis functions $\phi_{\mb l,\mb i}$ in hierarchical basis $W_{\mb l}$ and the dot represents the center of this support, where $\phi_{\mb l,\mb i}$ takes maximum value $1$. The basis functions $\phi_{\mb l,\mb i}$ are the tensor product of the univariate functions represented next to the two axis. The approximation space of sparse grids $V_3^{(1)}$ is the span of the hierarchical basis below the dashed line ($|\mb l|_1\leq 4$). Instead of taking the complete grid ($|\mb l|_\infty\leq 4$) sparse grids uses a significantly smaller approximation space while conserving similar accuracy.}
    \label{fig:sparse grids}
\end{figure}
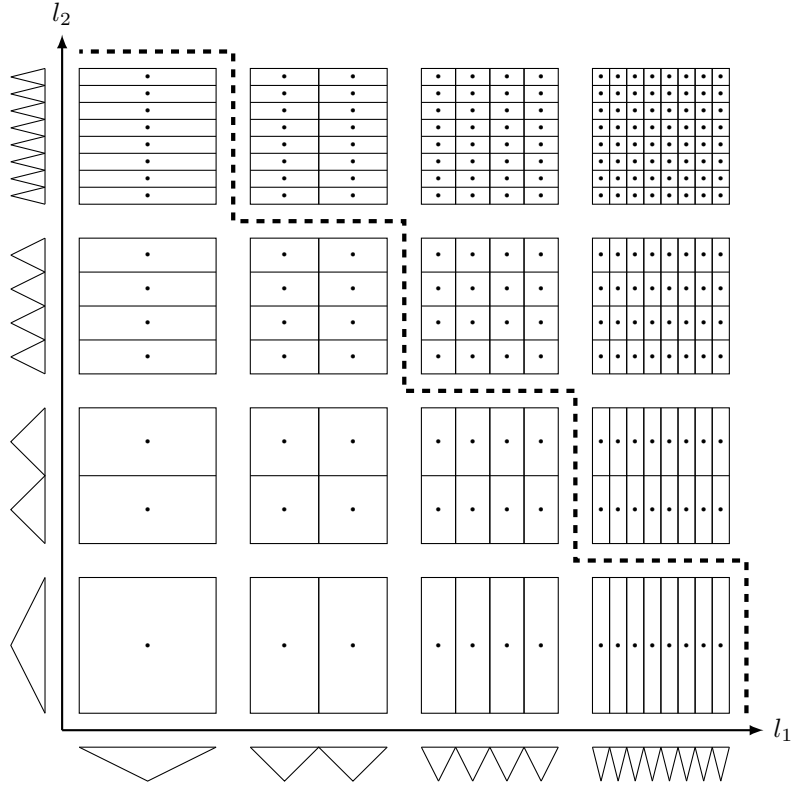
For the particular choice of the hat function for $\phi$, Bungartz and Griebel \cite{bungartz2004sparse} showed that this choice of approximating space leads to a good approximating error.
\begin{theorem}[Bungartz and Griebel \cite{bungartz2004sparse}]\label{thm: sparse grid approx.}
Let $f \in X^{2,\infty}(\Omega)$
% be a function defined on $[0,1]^d$ vanishing at the boundary and having bounded mixed derivatives up to order $2$. Let 
and $f_n^{(1)}$ be the projection of $f$ on the subspace $V_n^{(1)}$. We have,
\begin{equation*}
    \|f-f_n^{(1)}\|_\infty = O\left(2^{-2n} n^{d-1}\right).
\end{equation*}
More precisely, 
\begin{equation*}
    \|f-f_n^{(1)}\|_\infty \leq \frac{2\cdot |f|_{\mb 2,\infty}}{8^d}\cdot 2^{-2n}\cdot A(d,n),
\end{equation*}
where
\begin{equation*}
    A(d,n):= \sum_{k=0}^{d-1}\binom{n+d-1}{k} = \frac{n^{d-1}}{(d-1)!}+ O(n^{d-2}).
\end{equation*}
Furthermore, if $v_{\mb l,\mb i}$ denotes the coefficient of $\phi_{\mb l,\mb i}$ in the decomposition of $f_n^{(1)}$ in $V^{(1)}_n$, then
\begin{equation*}
    |v_{\mb l,\mb i}|\leq 2^{-d}\cdot c^{-2|\mb l|_1}\cdot |f|_{\mb 2,\infty},
\end{equation*}
for all $\mb{l},\mb{i} \in \mathbb{N}^d$ with $|\mb l|_1\leq n+d-1$, $\mb 1\leq \mb i\leq 2^{\mb l} -\mb 1$ where $\mb i$ has odd components.
\end{theorem}

\section{The Representation Power of Shallow Neural Networks}\label{sec: shallow nn.}
It has recently been shown that deep neural networks, with depth scaling with $\epsilon$, break the curse of dimensionality on the numbers of neuron needed to approximate the Korobov space \cite{montanelli2019new}. However, the question of whether shallow neural networks with fixed universal depth --- independent of $\epsilon$ and $d$ --- escape the curse of dimensionality as well for the Korobov space remains open. We settle this question by proving that shallow neural networks also break the curse of dimensionality for Korobov space.

\begin{theorem}\label{thm: shallow NN bound.}
Let $\epsilon>0$. For all $f \in X^{2,\infty}(\Omega)$, there exists a neural network with 2 layers, ReLU activation, 
$O(\epsilon^{-1} (\log \frac{1}{\epsilon})^{\frac{3d-1}{2}})$
neurons, and $O(\epsilon^{-\frac{1}{2}}(\log \frac{1}{\epsilon})^{\frac{3(d-1)}{2}})$ training parameters that approximates $f$ within $\epsilon$ for the infinity norm.
\end{theorem}
In order to prove Theorem \ref{thm: shallow NN bound.}, we construct the approximating neural network explicitly throughout the remaining of this section.

\subsection{Approximating the product function}
\label{section:product approximation}
The first step is to construct a neural network architecture with two layers and  $O(d^{3/2}\epsilon^{-1/2} \log\frac{1}{\epsilon})$ neurons that approximates the product function $p: \mb{x}\in [0,1]^d \longmapsto \prod_{i=1}^n x_i$ within $\epsilon$ for all $\epsilon>0$. We first prove a simple lemma to represent univariate piece-wise affine functions by shallow neural networks.

\begin{lemma}
\label{lemma:piece-wise affine}
Any one dimensional continuous piece-wise affine function with $m$ pieces is representable exactly by a shallow neural network with ReLU activation, with $m$ neurons on a single layer.
\end{lemma}
\begin{proof}
Let $x_1\leq \cdots \leq x_{m-1}$ be the subdivision of the piece-wise affine function $f$. We use a neural network of the form
\begin{equation*}
   g(x):=f(x_1) + \sum_{k=1}^{m-1} w_k (x-x_k)_+ - w_0 (x_1-x)_+ ,
\end{equation*}
where $w_0$ is the slope of $f$ on the piece $\leq x_1$, $w_1$ is the slope of $f$ on the piece $[x_1,x_2]$,
\begin{equation*}
    w_k = \frac{f(x_{k+1}) - f(x_1) - \sum_{i=1}^{k-1} w_i (x_{k+1}-x_i)}{x_{k+1}-x_k},
\end{equation*}
for $k=1,\cdots,m-2$, and $w_{m-1} = \tilde w -\sum_{k=1}^{m-2} w_k$ where $\tilde w$ is the slope of $f$ on the piece $\geq x_{m-1}$. Notice that $f$ and $g$ coincide on all $x_k$ for $1\leq k\leq m-1$. Furthermore, $g$ has same slope as $f$ on each pieces, therefore, $g=f$.
\end{proof}

We can approximate univariate functions right continuous functions by piece-wise affine functions, and then use Lemma \ref{lemma:piece-wise affine} to represente them by shallow neural networks. The following lemma shows that $O(\epsilon^{-1})$ neurons are sufficient to represent an increasing right-continuous function with a shallow neural network.

\begin{restatable}{lemma}{LemmaOneDFunctionApprox}
\label{lemma:1D function approx.}
Let $f:I \longrightarrow [c,d]$ be a right-continuous increasing function where $I$ is an interval, and let $\epsilon>0$. There exists a shallow neural network with ReLU activation, with $\left\lceil\frac{d-c}{\epsilon}\right\rceil$ neurons on a single layer, that approximates $f$ within $\epsilon$ for the infinity norm.
\end{restatable}
\begin{proof}
See Appendix \ref{appendix sec: NN upper bound sec}.
\end{proof}

If the function to approximate has some regularity, the number of neurons needed for approximation can be significantly reduced. In the following lemma, we show that $O(\epsilon^{-1/2})$ neurons are sufficient to approximate a $\mathcal C^2$ univariate function with a shallow neural network.

\begin{restatable}{lemma}{LemmaQuadraticApproximation}
\label{lemma:quadratic approximation 1D}
Let $f:[a,b] \longrightarrow [c,d]\in \mathcal C^2$, and let $\epsilon>0$. There exists a shallow neural network with ReLU activation, with $\frac{1}{\sqrt{2\epsilon}}\min(\int \sqrt{ |f''|}(1+\mu(f,\epsilon)), (b-a) \sqrt{\|f''\|_{\infty}})$ neurons on a single layer, where $\mu(f,\epsilon)\to1$ as $\epsilon\to 0$, that approximates $f$ within $\epsilon$ for the infinity norm.
\end{restatable}

\begin{proof}
See Appendix \ref{appendix sec: NN upper bound sec}.
\end{proof}

We will now use the ideas of Lemma \ref{lemma:1D function approx.} and Lemma \ref{lemma:quadratic approximation 1D} to approximate a truncated $\log$ function, which we will use in the construction of our neural network approximating the product.

% \begin{corollary}{CorQuadraticApproximationLog}
\begin{restatable}{corollary}{CorQuadraticApproximationLog}
\label{corollary:quadratic approximation log}
Let $\epsilon>0$ sufficiently small and $\delta>0$. Consider the truncated logarithm function $\log:[\delta,1] \longrightarrow \mathbb R$. There exists a shallow neural network with ReLU activation, with $\epsilon^{-1/2} \log\frac{1}{\delta}$ neurons on a single layer, that approximates $f$ within $\epsilon$ for the infinity norm.
% \end{corollary}
\end{restatable}
% \begin{proof}
% In view of Lemma \ref{lemma:quadratic approximation 1D}, the goal is to show that we can remove the dependence of $\mu(f,\epsilon)$ in $\delta$. This essentially comes from the fact that the upper Riemann sum behaves well for approximating $\log$. Consider the subdivision $x_0:=\delta\leq x_1 \leq \cdots \leq x_m \leq x_{m+1}:=1$ with $m=\left \lfloor\frac{1}{\tilde \epsilon}\log\frac{1}{\delta}\right\rfloor$ where $\tilde \epsilon:= \log(1+\sqrt{2\epsilon})$, such that $ x_k = e^{\log \delta + k\tilde \epsilon}$, for $k=0,\cdots,m-1$. Denote $\hat f$ the corresponding piece-wise affine approximation. Similarly to the proof of Lemma \ref{lemma:quadratic approximation 1D}, for $k=0,\cdots,m-1$,
% \begin{equation*}
%     \|\log-\hat f\|_{\infty, [x_k,x_{k+1}]} \leq \frac{1}{2} (x_{k+1}-x_k)^2 \|f''\|_{\infty, [x_k,x_{k+1}]}\leq \frac{(e^{\tilde \epsilon}-1)^2}{2}\leq \epsilon.
% \end{equation*}
% The proof follows.
% \end{proof}

We are now ready to construct a neural network approximating the product function. The proof builds upon on the observation that $\prod_{i=1}^d x_i = \exp(\sum_{i=1}^d \log x_i).$ We construct an approximating two-layer neural network where the first layer computes $\log x_i$ for $1\leq i\leq d$, and the second layer computes the exponential. We illustrate the construction of the proof in Figure \ref{fig:product approx}.
\begin{proposition}
\label{prop:product_approx}
For all $\epsilon>0$, there exists a neural network with depth $2$, ReLU activation and $O(d^{3/2}\epsilon^{-1/2} \log \frac{1}{\epsilon})$ neurons, that approximates the product function $p: \mb{x}\in [0,1]^d \longrightarrow \prod_{i=1}^d x_i$ within $\epsilon$ for the infinity norm.
\end{proposition}

\begin{proof}
Fix $\epsilon>0$. Consider the function $h_{\epsilon}: x\in[0,1] \mapsto \max(\log x, \; \log \epsilon)\in [\log \epsilon , 0]$. Using Corollary \ref{corollary:quadratic approximation log}, there exists a neural network $\hat{h}_{\epsilon}: [0,1] \longrightarrow [\log \epsilon , 0]$ with $1+\lceil d^{1/2}\epsilon^{-1/2} \log \frac{1}{\epsilon} \rceil $ neurons on a single layer such that $\| h_{\epsilon} - \hat{h}_{\epsilon} \| _{\infty} \leq \frac{\epsilon}{d}$. Indeed, one can take the $\epsilon-$approximation of $h_\epsilon: x\in [\epsilon,1] \mapsto \log x\in  [\log\epsilon, 0]$, then extend this function to $[0,\epsilon]$ with a constant equal to $\log\epsilon$. The resulting piece-wise affine function has one additional segment corresponding to one additional neuron in the approximating function.
Similarly, consider the exponential $g: x\in\mathbb R_- \mapsto e^x\in [0, 1]$. Because $g$ is right-continuous increasing, we can use Lemma \ref{lemma:quadratic approximation 1D} to construct a neural network $\hat{g}_{\epsilon}: \mathbb R_- \longrightarrow [0, 1]$ with $1+\left\lceil \frac{1}{\sqrt{2\epsilon}}\log\frac{1}{\epsilon} \right\rceil $ neurons on a single layer such that $\| g - \hat{g}_{\epsilon} \| _{\infty} \leq \epsilon$. Indeed, again one can take the $\epsilon-$approximation of $g_\epsilon: x\in [\log\epsilon,0] \mapsto e^x\in [0, 1]$, then extend this function to $(-\infty, \log\epsilon]$ with a constant equal to $\epsilon$. The corresponding neural network has an additional neuron. We construct our final neural network $\hat{\phi}_{\epsilon}$ (see Figure \ref{fig:product approx}) as
\begin{equation*}
    \hat{\phi}_{\epsilon}=\hat{g}_{\epsilon}\left(\sum_{i=1}^d \hat{h}_{\epsilon}(x_i)\right).
\end{equation*}
Note that $\hat{\phi}_{\epsilon}$ can be represented as a two layer neutral network: the first layer is composted of the union of the $1+\lceil d^{1/2}\epsilon^{-1/2} \log \frac{1}{\epsilon} \rceil$ neurons composing each of the one-layer neural networks $\hat{h}^i_{\epsilon}: \mb x \in[0,1]^d \mapsto \hat{h}_{\epsilon}(x_i)\in \mathbb{R}$ for each dimension $i\in \{1,\ldots,d\}$. The second layer is composed of the $1+\left\lceil \frac{1}{\sqrt{2\epsilon}}\log\frac{1}{\epsilon} \right\rceil$ neurons of $\hat{g}_{\epsilon}$. Hence, the constructed neural network
$\hat{\phi}_{\epsilon}$ has $O(d^{3/2}\epsilon^{-1/2} \log \frac{1}{\epsilon}) $ neurons. Let us now analyze the approximation error. Let $\mb{x} \in [0,1]^d$. For the sake of brevity, denote $\hat y=\sum_{i=1}^d \hat{h}_{\epsilon}(x_i)$ and $y = \sum_{i=1}^d \log(x_i)$. We have,
\begin{equation*}
    | \hat{\phi}_{\epsilon}(\mb{x}) - p(\mb{x})|
    \leq |\hat{\phi}_{\epsilon}(\mb{x}) -\exp (\hat y)|+|\exp (\hat y) - \exp(y)| \leq \epsilon + \prod_{i=1}^d x_i\cdot |\exp (\hat y-y) -1|,
\end{equation*}
where we used the fact that $|\hat{\phi}_{\epsilon}(\mb{x}) - \exp(\hat y)| = |\hat{g}_{\epsilon}(\hat y) - g(\hat y)|\leq \|\hat{g}_{\epsilon} - g\|_{\infty} \leq \epsilon$.

First suppose that $\mb x\geq \epsilon$. In this case, for all $i\in \{1,\ldots,d\}$ we have $|\hat h_\epsilon(x_i)-\log(x_i)| = |\hat h_\epsilon(x_i)-h_\epsilon(x_i)|\leq \frac{\epsilon}{d}$. Then, $|\hat y - y|\leq \epsilon$. Consequently, $|\hat{\phi}_{\epsilon}(\mb{x}) - p(\mb{x})|\leq \epsilon +\max(|e^\epsilon-1|, |e^{-\epsilon}-1|)\leq 3\epsilon$, for $\epsilon>0$ sufficiently small. Without loss of generality now suppose $x_1\leq \epsilon$. Then $\hat y \leq h_\epsilon(x_1)\leq \log \epsilon$, so by definition of $\hat g_\epsilon$, we have $0\leq\hat\phi_\epsilon(\mb x) = \hat g_\epsilon(\hat y)\leq \exp(\log \epsilon)=\epsilon$. Also, $0\leq p(x)\leq \epsilon$ so finally $|\hat{\phi}_{\epsilon}(\mb{x}) - p(\mb{x})|\leq \epsilon$.
\end{proof}

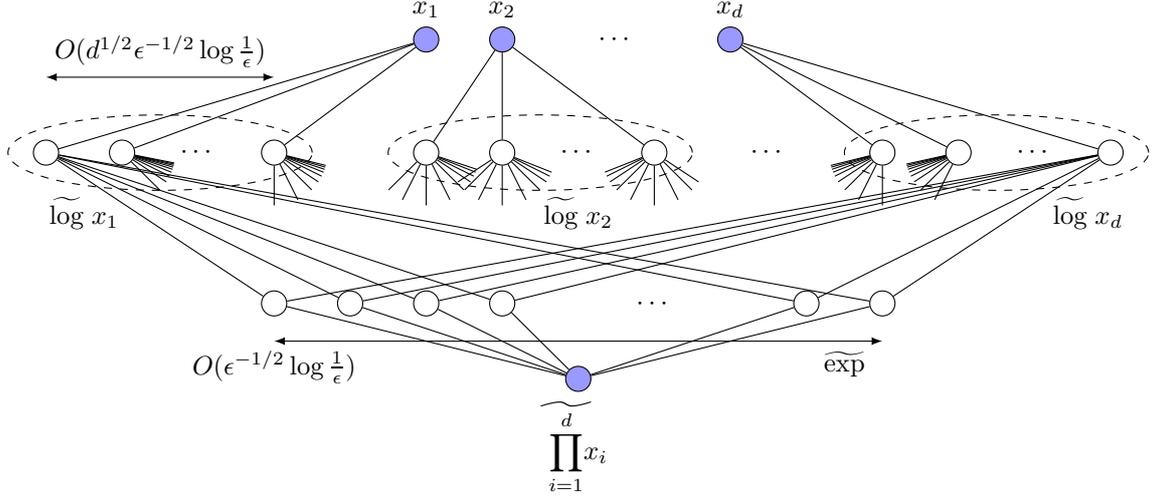
\begin{figure}
    \centering
    \begin{tikzpicture}
    [   cnode/.style={draw=black,fill=#1,minimum width=3mm,circle},
    ]
    %input neurons
	\node[cnode=blue!40, label=above:$x_1$] (x1) at (-2,3.5) {};    	
	\node[cnode=blue!40, label=above:$x_2$] (x2) at (-1,3.5) {}; 
	\draw (0.5,3.5) node {$\cdots$};
	\node[cnode=blue!40, label=above:$x_d$] (xd) at (2,3.5) {}; 
	
	%1st layer neurons
	\node[cnode=white] (y-7) at (-7,2) {}; 
	\node[cnode=white] (y-6) at (-6,2) {};  
	\draw (-5,2) node {$\cdots$};
	\node[cnode=white] (y-4) at (-4,2) {}; 
	
	\node[cnode=white] (y-1) at (-1,2) {}; 
	\node[cnode=white] (y-2) at (-2,2) {};  
	\draw (0,2) node {$\cdots$};
	\node[cnode=white] (y1) at (1,2) {};
	
	\draw (2.5,2) node {$\cdots$}; 
	
	\node[cnode=white] (y4) at (4,2) {}; 
	\node[cnode=white] (y5) at (5,2) {};  
	\draw (6,2) node {$\cdots$};
	\node[cnode=white] (y7) at (7,2) {};
	
	%edges from input to 1st layer
	\foreach \y in {-7,-6,-4}
	\draw (x1) -- (y\y);
	
	\foreach \y in {-2,-1,1}
	\draw (x2) -- (y\y);
	
	\foreach \y in {4,5,7}
	\draw (xd) -- (y\y);
	
	%2nd layer neurons
	\node[cnode=white] (z-4) at (-4,0) {};
	\node[cnode=white] (z-3) at (-3,0) {};
	\node[cnode=white] (z-2) at (-2,0) {};
	\node[cnode=white] (z-1) at (-1,0) {};
	\draw node (z0) at (0,0) {};
	\draw node (z1) at (1,0) {};
	\draw node (z2) at (2,0) {};
	\draw (1,0) node {$\cdots$};
	\node[cnode=white] (z4) at (4,0) {};
	\node[cnode=white] (z3) at (3,0) {};
	
	%edges from 1st to 2nd layer
	\foreach \y in {-6,-4,-2,-1,1,4,5}{
	    \foreach \z in {-4,-3,-2,-1,0,1,2,3,4}
	        \draw (y\y) -- ($(y\y)!0.7cm!(z\z)$);
	}
	\foreach \z in {-4,-3,-2,-1,3,4}{
	    \draw (y-7) -- (z\z);
	    \draw (y7) -- (z\z);
	}
	%output node
	\node[cnode=blue!40, label=below:$\displaystyle \widetilde{
        \prod_{i=1}^d} x_i$] (output) at (0,-1) {}; 
	
	%edges to output
	 \foreach \z in {-4,-3,-2,-1,3,4}
	       \draw (z\z) -- (output);

	%explanations
	\draw [dashed] (-5.5,2) ellipse (2cm and 0.5cm);
	\draw (-6.5,1.2) node {$\widetilde \log\; x_1$};
	\draw [<->,>=latex] (-7,3) -- node[above] {$O(d^{1/2}\epsilon^{-1/2}\log\frac{1}{\epsilon})$} (-4,3);
	
	\draw [dashed] (-0.5,2) ellipse (2cm and 0.5cm);
	\draw (0,1.2) node {$\widetilde \log\; x_2$};
	
	\draw [dashed] (5.5,2) ellipse (2cm and 0.5cm);
	\draw (6.7,1.2) node {$\widetilde \log\; x_d$};
	
	%\draw [dashed] (0,0) ellipse (4.5cm and 0.5cm);
	\draw [<->,>=latex] (-4,-0.5) node[below] {$O(\epsilon^{-1/2}\log\frac{1}{\epsilon})$} --  (4,-0.5) ;
	\draw (3.5,-0.8) node {$\widetilde \exp$};
    
    \end{tikzpicture}
    \caption{Shallow neural network with ReLU activation implementing the product function $\prod_{i=1}^d x_i$ within $\epsilon$ in infinity norm. The network has $O(d^{3/2}\epsilon^{-1/2} \log \frac{1}{\epsilon})$ neurons on the first layer and $O(\epsilon^{-1/2}\log\frac{1}{\epsilon})$ neurons on the second layer.}
    \label{fig:product approx}
\end{figure}
\begin{remark}
\label{remark: for sigmoids}
Note that using Lemma \ref{lemma:1D function approx.} instead of Lemma \ref{lemma:quadratic approximation 1D} to construct approximating shallow networks for $\log$ and $\exp$ would yield approximation functions $\hat h_\epsilon$ with $O(\left \lceil\frac{d}{\epsilon}\log\frac{1}{\epsilon}\right\rceil)$ neurons and $\hat g_\epsilon$ with $O(\left \lceil\frac{1}{\epsilon}\right\rceil)$ neurons. Therefore, the corresponding neural network would approximate the product $p$ with $O(d^2\epsilon^{-1}\log \frac{1}{\epsilon})$ neurons.
\end{remark}

\subsection{Proof of Theorem \ref{thm: shallow NN bound.}:  Approximating the Korobov space $X^{2,\infty}(\Omega)$}

In this subsection, we prove Theorem \ref{thm: shallow NN bound.} and show that we can approximate any Korobov function $f \in X^{2,\infty}(\Omega)$ within $\epsilon$ with a 2-layer neural network of $O( \epsilon^{-\frac{1}{2}}( \log \frac{1}{\epsilon})^{\frac{3(d-1)}{2}})$ neurons. Our proof combines the constructed network approximating the product function and a decomposition of $f$ as a sum of separable functions, i.e. a decomposition of the form
\begin{equation*}
    f(\mb x) \approx \sum_{k=1}^K \prod_{j=1}^d \phi_j^{(k)} (x_j), \quad \forall \mb{x} \in [0,1]^d.
\end{equation*}

Consider the sparse grid construction of the approximating space $V_n^{(1)}$ using the standard hat function \eqref{eq: hat function} as mother function to create the hierarchical basis $W_{\mb l}$ (introduced in Section \ref{subsection: sparse grids}). We recall that the approximation space is defined as $V_n^{(1)} := \bigoplus_{|\mb l|_1\leq n+d-1} W_{\mb l}$. We will use the result of Theorem \ref{thm: sparse grid approx.} on the approximation error for sparse grids to show that shallow neural networks can approximate functions with bounded second derivatives efficiently. Let $f_n^{(1)}$ be the projection of $f$ on the subspace $V_n^{(1)}$ defined in \eqref{eq: V1.}. $f_n^{(1)}$ can be written as
\begin{equation*}
    f_n^{(1)}(\mb x)=\sum_{(\mb l, \mb i)\in U_n^{(1)}} v_{\mb l,\mb i} \phi_{\mb l,\mb i}(\mb x),
\end{equation*}
where $U_n^{(1)}$ contains the indices $(\mb l, \mb i)$ of basis functions present in $V_n^{(1)}$ i.e.
\begin{equation}\label{eq: Un1.}
    U_n^{(1)}:= \{ (\bold l,\bold i), \quad |\bold l|_1\leq n+d-1,\; \bold 1\leq \bold i\leq 2^{\bold l}-\bold 1,\; i_j \text{ odd for all }1\leq j\leq d\}.
\end{equation}
Throughout the proof, we explicitly construct a neural network that uses this decomposition to approximate $f_n^{(1)}$. We then use Theorem \ref{thm: sparse grid approx.} and choose $n$ carefully such that $f_n^{(1)}$ approximates $f$ within $\epsilon$ for $L^{\infty}$ norm. Note that the basis functions can be written as a product of univariate functions $\phi_{\mb l,\mb i} = \prod_{j=1}^d \phi_{l_j,i_j}$. We can therefore use the product approximation of Proposition \ref{prop:product_approx} to approximate the basis functions. Specifically, we will use one layer to approximate the terms $\log \phi_{l_j,i_j}$ and a second layer to approximate the exponential.

We now present in detail the construction of the first layer. First, recall that $\phi_{l_j,i_j}$ is a piece-wise affine function with subdivision $0\leq \frac{i_j-1}{2^{l_j}}\leq \frac{i_j}{2^{l_j}}\leq \frac{i_j+1}{2^{l_j}}\leq 1.$ Define the error term $\tilde \epsilon := \frac{\epsilon}{2|f|_{\bold 2,\infty}}$. We consider a symmetric subdivision of the interval $\left[\frac{i_j-1+\tilde \epsilon}{2^{l_j}},\frac{i_j+1-\tilde\epsilon}{2^{l_j}}\right]$. We define it as follows: $x_0=\frac{i_j-1+\tilde \epsilon}{2^{l_j}}\leq x_1\leq \cdots \leq x_{m+1}=\frac{i_j}{2^{l_j}}\leq x_{m+2}\leq \cdots \leq x_{2m+2}=\frac{i_j+1-\tilde \epsilon}{2^{l_j}}$ where $m=\left \lfloor\frac{1}{\epsilon_0}\log\frac{1}{\tilde \epsilon}\right\rfloor$ and $\epsilon_0:=\log(1+\sqrt{2\tilde\epsilon/d})$, such that
\begin{align*}
    x_k &= \frac{i_j+ e^{\log \tilde \epsilon + k\epsilon_0}}{2^{l_j}}\quad \quad \quad\quad\quad\quad\; 0\leq k\leq m,\\
    x_k &=\frac{i_j+1 - e^{\log \tilde \epsilon + (2m+2-k)k\epsilon_0}}{2^{l_j}}\quad  m+2\leq k\leq 2m+2.
\end{align*}
Note that with this definition, the terms $\log (2^{l_j} x_k-i_j)$ form a regular sequence with step $\epsilon_0.$ We now construct the piece-wise affine function $\hat g_{l_j,i_j}$ on the subdivision $x_0\leq \cdots \leq x_{2m+2}$ which coincides with $\log\phi_{l_j,i_j}$ on $x_0,\cdots,x_{2m+2}$ and is constant on $[0,x_0]$ and $[x_{2m+2},1].$ By Lemma \ref{lemma:piece-wise affine}, this function can be represented by a single layer neural network with as much neurons as the number of pieces of $\hat g$, i.e. at most $2\sqrt{\frac{3d}{\tilde\epsilon}}\log\frac{1}{\tilde \epsilon}$ neurons for $\epsilon$ sufficiently small. A similar proof to that of Corollary \ref{corollary:quadratic approximation log} shows that $\hat g$ approximates $\max(\log \phi_{l_j,i_j},\log (\tilde\epsilon/3))$ within $\tilde\epsilon/(3d)$ for the infinity norm. We use this construction to compute in parallel, $\tilde\epsilon/(3d)-$approximations of $\max(\log \phi_{l_j,i_j}(x_j),\log \tilde\epsilon)$ for all $1\leq j\leq d$, and $1\leq l_j\leq n$, $1\leq i_j\leq 2^{l_j}$ where $i_j$ is odd. These are exactly the $1-$dimensional functions that we will need, in order to compute the $d-$dimensional function basis of the approximation space $V_n^{(1)}$. There are $d(2^n-1)$ such univariate functions, therefore our first layer contains at most $2^{n+1}d\sqrt{\frac{3d}{\tilde\epsilon}}\log\frac{1}{\tilde \epsilon}$ neurons.

\comment{Note that this first layer differs from the first layer that we would obtain by computing in parallel the basis functions
\begin{equation*}
    \phi_{\bold l, \bold i} = \prod_{j=1}^d \phi_{l_j,i_j}(x_j).
\end{equation*}
Indeed, this would ask to compute $\max(\log \phi_{l_j,i_j}(x_j),\log \tilde\epsilon)$ for $1\leq j\leq d$. However, this is redundant since an approximation of $\max(\log \phi_{l_j,i_j}(x_j),\log \tilde\epsilon)$ would be computed for all $\bold {\tilde l},\bold {\tilde i}$ such such that $\tilde l_j=l_j$ and $\tilde i_j=i_j.$ Here we avoid this redundancy by using the same approximation of $\max(\log \phi_{l_j,i_j}(x_j),\log \tilde\epsilon)$ for several basis functions $\phi_{\bold {\tilde l},\bold {\tilde i}}$.\\
}

We now turn to the second layer. The result of the first two layers will be $\tilde \epsilon/3-$approximations of $\phi_{\bold l,\bold i}$ for all $(\bold l,\bold i)\in U_n^{(1)}$. Recall that $U_n^{(1)}$ contains the indices for the functions forming a basis of the approximation space $V_n^{(1)}$. To do so, for each indexes $(\bold l,\bold i)\in U_n^{(1)}$ we construct a single-layer neural network approximating the function $\exp$, which will compute an approximation of $\exp(\hat g_{l_1,i_1} + \cdots + \hat g_{l_d,i_d})$. The approximation of $\exp$ is constructed in the same way as for Lemma \ref{lemma:quadratic approximation 1D}. Consider a regular subdivision of the interval $[\log(\tilde\epsilon/3),0]$ with step $\sqrt{2(\tilde\epsilon/3)}$, i.e. $x_0:= \log(\tilde\epsilon/3)\leq x_1\leq \cdots \leq x_m\leq x_{m+1}=0$ where $m = \left\lfloor\sqrt{\frac{3}{2\tilde\epsilon}}\log\frac{3}{\tilde\epsilon}\right\rfloor$, such that $x_k = \log\tilde\epsilon + k\sqrt{2\tilde\epsilon},\quad 0\leq k\leq m.$ Construct the piece-wise affine function $\hat h$ on the subdivision $x_0\leq \cdots\leq x_{m+1}$ which coincides with $\exp$ on $x_0,\cdots,x_{m+1}$ and is constant on $(-\infty,x_0]$. Lemma \ref{lemma:quadratic approximation 1D} shows that $\hat h$ approximates $\exp$ on $\R_-$ within $\tilde \epsilon$ for the infinity norm. Again, Lemma \ref{lemma:piece-wise affine} gives a representation of $\hat h$ as a single layer neural network with as many neurons as pieces in $\hat h$ i.e. $1+\left\lceil\sqrt{\frac{3}{2\tilde\epsilon}}\log\frac{3}{\tilde\epsilon}\right\rceil$. The second layer is the union of single-layer neural networks approximating $\exp$ within $\tilde \epsilon/3$, for each indexes $(\bold l,\bold i)\in U_n^{(1)}$. Therefore, the second layer contains $\left| U_n^{(1)}\right|\left(1+\left\lceil \sqrt{\frac{3}{2\tilde\epsilon}}\log\frac{3}{\tilde\epsilon}\right\rceil\right)$ neurons. As shown in \cite{bungartz2004sparse},
\begin{equation*}
   \left| U_n^{(1)}\right| = \sum_{i=0}^{n-1} 2^i \cdot \binom{d-1+i}{d-1} = (-1)^d + 2^n \sum_{i=0}^{d-1}\binom{n+d-1}{i} (-2)^{d-1-i} = 2^n\cdot  \left(\frac{n^{d-1}}{(d-1)!}+ O(n^{d-2})\right).
\end{equation*}
Therefore, the second layer has $O\left(2^n \frac{n^{d-1}}{(d-1)!}\tilde\epsilon^{-1/2}\log\frac{1}{\tilde\epsilon}\right)$ neurons. Finally, the output layer computes the weighted sum of the basis functions to approximate $f_n^{(1)}$. Denote by $\hat{f}_n^{(1)}$ the corresponding function of the constructed neural network (see Figure \ref{fig:shallow nn}), i.e.
\begin{equation*}
    \hat{f}_n^{(1)} = \sum_{(\bold l,\bold i)\in U_n^{(1)}} v_{\bold l,\bold i}\cdot \hat h\left(\sum_{j=1}^d \hat g(x_j)\right).
\end{equation*}

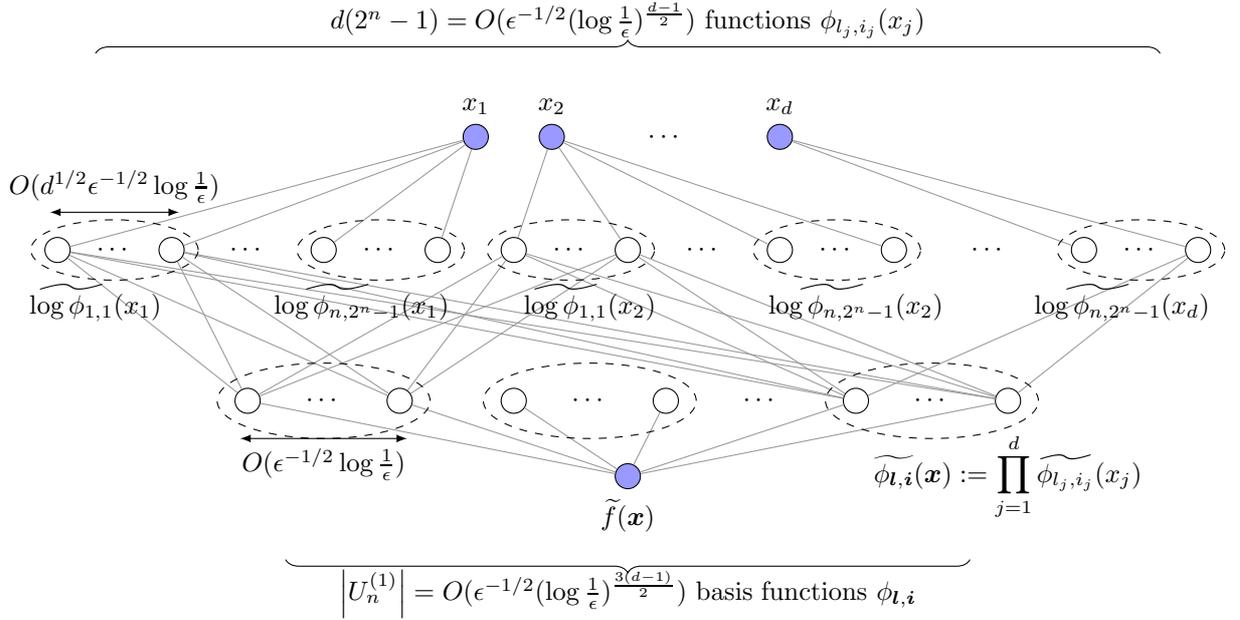
\begin{figure}
    \centering
    \begin{tikzpicture}
    [   cnode/.style={draw=black,fill=#1,minimum width=3mm,circle},
    ]
    %input neurons
	\node[cnode=blue!40, label=above:$x_1$] (x1) at (-2,3.5) {};    	
	\node[cnode=blue!40, label=above:$x_2$] (x2) at (-1,3.5) {}; 
	\draw (0.5,3.5) node {$\cdots$};
	\node[cnode=blue!40, label=above:$x_d$] (xd) at (2,3.5) {}; 
	
	%1st layer neurons
	%nodes incident to x1
	\node[cnode=white] (y-75) at (-7.5,2) {};  
	\draw (-6.75,2) node {$\cdots$};
	\node[cnode=white] (y-6) at (-6,2) {}; 
	
	\draw (-5,2) node {$\cdots$};
	
	\node[cnode=white] (y-4) at (-4,2) {};  
	\draw (-3.25,2) node {$\cdots$};
	\node[cnode=white] (y-25) at (-2.5,2) {};
	
	%nodes incident to x2
	\node[cnode=white] (y-15) at (-1.5,2) {};  
	\draw (-0.75,2) node {$\cdots$};
	\node[cnode=white] (y0) at (0,2) {}; 
	
	\draw (1,2) node {$\cdots$};
	
	\node[cnode=white] (y2) at (2,2) {};  
	\draw (2.75,2) node {$\cdots$};
	\node[cnode=white] (y35) at (3.5,2) {};
	
	\draw (4.75,2) node {$\cdots$}; 
	%nodes icident to xd
	\node[cnode=white] (y6) at (6,2) {};  
	\draw (6.75,2) node {$\cdots$};
	\node[cnode=white] (y75) at (7.5,2) {};
	
	%edges from input to 1st layer
	\draw   [gray!80] (x1) -- (y-75);
	\draw   [gray!80] (x1) -- (y-6);
	\draw   [gray!80] (x1) -- (y-4);
	\draw   [gray!80] (x1) -- (y-25);
	
	\draw   [gray!80] (x2) -- (y-15);
	\draw   [gray!80] (x2) -- (y0);
	\draw   [gray!80] (x2) -- (y2);
	\draw   [gray!80] (x2) -- (y35);
	
	\draw   [gray!80] (xd) -- (y6);
	\draw   [gray!80] (xd) -- (y75);
	
	%2nd layer neurons
	\node[cnode=white] (z-5) at (-5,0) {};
	\draw (-4,0) node {$\cdots$};
	\node[cnode=white] (z-3) at (-3,0) {};
	
	\node[cnode=white] (z-15) at (-1.5,0) {};
	\draw (-0.5,0) node {$\cdots$};
	\node[cnode=white] (z05) at (0.5,0) {};
	
	\draw (1.75,0) node {$\cdots$};
	
	\node[cnode=white] (z3) at (3,0) {};
	\draw (4,0) node {$\cdots$};
	\node[cnode=white] (z5) at (5,0) {};
	\draw node (z0) at (0,0) {};
	\draw node (z1) at (1,0) {};
	\draw node (z2) at (2,0) {};

	%edges from 1st to 2nd layer
	
	\foreach \y in {-75,-6,-15,0}{
	    \foreach \z in {-5,-3}
	        \draw   [gray!80] (y\y) -- (z\z);
	}
	
	\foreach \y in {-75,-6,-15,0,75}{
	    \foreach \z in {3,5}
	        \draw   [gray!80] (y\y) -- (z\z);
	}
	%$(y\y)!0.7cm!(z\z)$
	
	%output node
	\node[cnode=blue!40, label=below:$\displaystyle \widetilde{
        f}(\mb x)$] (output) at (0,-1) {}; 
	
	%edges to output
	 \foreach \z in {-5,-3,-15,05,3,5}
	       \draw   [gray!80] (z\z) -- (output);

	%explanations
	\foreach \x in {-6.75,-3.25,-0.75,2.75,6.75}
	\draw [dashed] (\x,2) ellipse (1.1cm and 0.4cm);
	\draw (-7,1.3) node {$\widetilde{\log\phi_{1,1}}(x_1)$};
	\draw (-3.5,1.3) node {$\widetilde{\log\phi_{n,2^n-1}}(x_1)$};
	
	\draw (-0.5,1.3) node {$\widetilde{\log\phi_{1,1}}(x_2)$};
	\draw (3,1.3) node {$\widetilde{\log\phi_{n,2^n-1}}(x_2)$};
	
	\draw (6.5,1.3) node {$\widetilde{\log\phi_{n,2^n-1}}(x_d)$};
	
	\foreach \x in {-4,-0.5,4}
	\draw [dashed] (\x,0) ellipse (1.4cm and 0.5cm);
	
	%number of neurons at each step
	\draw [<->,>=latex] (-7.6,2.5) -- node[above] {$O(d^{1/2}\epsilon^{-1/2}\log\frac{1}{\epsilon})$} (-5.9,2.5);
	
	\draw [<->,>=latex] (-5.1,-0.5) -- node[below] {$O(\epsilon^{-1/2}\log\frac{1}{\epsilon})$} (-2.9,-0.5);
	
	%description of basis functions
	\draw (5,-1) node {$\displaystyle \widetilde{\phi_{\mb l,\mb i}}(\mb x):= \prod_{j=1}^d \widetilde{\phi_{l_j,i_j}}(x_j)$};
	
	%description of 1st layer
	\draw [decorate,decoration={brace,amplitude=5pt,raise=4ex}]
  (-7,4) -- (7,4) node[midway,yshift=3em]{$d(2^n-1) = O(\epsilon^{-1/2}(\log\frac{1}{\epsilon})^{\frac{d-1}{2}})$ functions $\phi_{l_j,i_j}(x_j)$};
	
	%number of basis functions
    \draw [decorate,decoration={brace,amplitude=5pt,mirror,raise=4ex}]
  (-4.5,-1.5) -- (4.5,-1.5) node[midway,yshift=-3em]{$\left|U_n^{(1)}\right| = O(\epsilon^{-1/2}(\log\frac{1}{\epsilon})^{\frac{3(d-1)}{2}})$ basis functions $\phi_{\mb l,\mb i}$};
	
    \end{tikzpicture}
    \caption{Shallow neural network with ReLU activation approximating a Korobov function $f\in X^{2,\infty}(\Omega)$ within $\epsilon$ in infinity norm. The network has $O(\epsilon^{-1} (\log \frac{1}{\epsilon})^{\frac{d+1}{2}})$ neurons on the first layer and $O(\epsilon^{-1} \left(\log \frac{1}{\epsilon} \right)^{\frac{3d-1}{2}})$ neurons on the second layer.}
    \label{fig:shallow nn}
\end{figure}
Let us analyze the approximation error of our neural network. The proof of Proposition \ref{prop:product_approx} shows that the output of the two first layers $h\left(\sum_{j=1}^d \hat g(\cdot_j)\right)$ approximates $\phi_{\bold l,\bold i}$ within $\tilde \epsilon$. Therefore, we obtain $\|f_n^{(1)}-\hat{f}_n^{(1)}\|_\infty \leq  \tilde \epsilon \sum_{(\bold l,\bold i)\in U_n^{(1)} }|v_{\bold l, \bold i}|.$ We now use approximation bounds from Theorem \ref{thm: sparse grid approx.} on $f_n^{(1)}.$
\begin{equation*}
    \|f-\hat{f}_n^{(1)}\|_\infty \leq  \|f-f_n^{(1)}\|_\infty + \|f_n^{(1)}-\hat{f}_n^{(1)}\|_\infty \leq \frac{2\cdot |f|_{\bold 2,\infty}}{8^d}\cdot 2^{-2n}\cdot A(d,n) + \frac{\epsilon}{2|f|_{\bold 2,\infty}} \sum_{(\bold l,\bold i)\in U_n^{(1)} }|v_{\bold l, \bold i}|,
\end{equation*}
where
\begin{equation*}
    \sum_{(\bold l,\bold i)\in U_n^{(1)} }|v_{\bold l, \bold i}|\leq |f|_{\bold 2,\infty} 2^{-d} \sum_{i\geq 0} 2^{-i} \cdot \binom{d-1+i}{d-1} \leq |f|_{\bold 2,\infty}.
\end{equation*}
Let us now take $n_\epsilon = \min \left\{ n:\;\frac{2 |f|_{\bold 2,\infty}}{8^d} 2^{-2n} A(d,n)\leq \frac{\epsilon}{2}\right\}$. Then, using the above inequality shows that the neural network $\hat f_{n_\epsilon}^{(1)}$ approximates $f$ within $\epsilon$ for the infinity norm. We will now estimate the number of neurons in each layer of this network. Note that
\begin{equation}\label{eq: n epsilon}
    n_\epsilon \sim \frac{1}{2\log 2}\log \frac{1}{\epsilon},\quad \text{and}\quad
    2^{n_\epsilon} \leq \frac{4}{8^{\frac{d}{2}}(2\log 2)^{\frac{d-1}{2}} (d-1)!^{\frac{1}{2}}}\sqrt{\frac{|f|_{\bold 2,\infty}}{\epsilon}} \left(\log \frac{1}{\epsilon} \right)^{\frac{d-1}{2}} \cdot (1+o(1)).
\end{equation}
We can use the above estimates to show that the constructed neural network has at most $N_1$ (resp. $N_2$) neurons on the first (resp. second) layer where
\begin{align*}
    N_1 &\underset{\epsilon\to0}{\sim} \frac{8\sqrt{6} d^2 }{8^{\frac{d}{2}}(2\log 2)^{\frac{d-1}{2}} d!^{\frac{1}{2}}}\cdot\frac{|f|_{\bold 2,\infty}}{\epsilon} \left(\log \frac{1}{\epsilon} \right)^{\frac{d+1}{2}},\\
    N_2 &\underset{\epsilon\to0}{\sim} \frac{ 4\sqrt{3}d^{3/2}}{8^{d/2}(2\log 2)^{\frac{3(d-1)}{2}} d!^{3/2}}\cdot\frac{|f|_{\bold 2,\infty}}{\epsilon} \left(\log \frac{1}{\epsilon} \right)^{\frac{3d-1}{2}}.
\end{align*}

This proves the bound the number of neurons. Finally, to prove the bound on the number of training parameters of the network, notice that the only parameters of the network that depend on the function $f$ are the parameters corresponding to the weighs $\mb{v_{l,i}}$ of the sparse grid decomposition. This number is $|U_{n_{\epsilon}}^{(1)}| = O(2^{n_{\epsilon}}n_{\epsilon}^{d-1}) = O(\epsilon^{-\frac{1}{2}}(\log \frac{1}{\epsilon})^{\frac{3(d-1)}{2}})$.
\subsection{Generalization to general activation functions}
\label{subsection: generalizing activations}
The shallow network we constructed in Theorem \ref{thm: shallow NN bound.} uses the ReLU activation function. We show in this subsection that this result can be generalized to a larger class of functions which contains the most popular activation functions. We begin by introducing the two classes of activation functions that we will consider.

\begin{definition}
A \emph{sigmoid-like} activation function $\sigma:\R\to\R$ is a non-decreasing function having finite limits in $\pm \infty$. A \emph{ReLU-like} activation function $\sigma:\R\to\R$ is a function having a horizontal asymptote in $-\infty$ i.e. $\sigma$ is bounded in $\R_-$, and an affine (non-horizontal) asymptote in $+\infty$, i.e. there exists $b>0$ such that $\sigma(x)-bx$ is bounded in $\R_+$.
\end{definition}
These two classes require no regularity on the activation. In practice, a large majority of the commonly used activation functions fall into one of the two classes. Examples of sigmoid-like activations include the Heaviside, logistic, tanh, arctan and softsign activations, while ReLU-like activations include the ReLU, ISRLU, ELU and soft-plus activations. We show that for these two classes of activations, we can extend Theorem \ref{thm: shallow NN bound.} and show that shallow neural network break the curse of dimensionality. 

\begin{theorem}\label{thm: shallow NN bound generalized activation.}
For any approximation tolerance $\epsilon>0$, and for any $f \in X^{2,\infty}(\Omega)$ there exists a neural network with depth $2$ and $O(\epsilon^{-\frac{1}{2}}(\log \frac{1}{\epsilon})^{\frac{3(d-1)}{2}})$ training parameters that approximates $f$ within $\epsilon$ for the infinity norm,
\begin{itemize}
    \item with $O\left( \frac{1}{\epsilon} \log (\frac{1}{\epsilon}\right)^{\frac{3d-1}{2}})$ neurons for a ReLU-like activation
    \item with $O\left( \frac{1}{\epsilon^{3/2}} \log (\frac{1}{\epsilon}\right)^{\frac{3(d-1)}{2}})$ neurons for a sigmoid-like activation
\end{itemize}
\end{theorem}

Intuitively, a sigmoid-like (resp. ReLU-like) function is a function that resembles the Heaviside (resp. ReLU) function by zooming out along the $x$ (resp. $x$ and $y$) axis. Using this approximation, we can reduce the analysis of sigmoid-like (resp. ReLU-like) activations to the case of a Heaviside (resp ReLU) activation and obtain the following result. The following lemma formalizes this intuition.

\begin{restatable}{lemma}{LemmaApproxActivations}
\label{lemma:approx activations}
Let $\sigma$ be a sigmoid-like activation with limit $a$ (resp. $b$) in $-\infty$ (resp.$+\infty$). For any $\delta>0$ and error tolerance $\epsilon>0$, there exists a scaling $M>0$ such that $x\mapsto \frac{\sigma(Mx)}{b-a} - a$ approximates the Heaviside function within $\epsilon$ outside of $(-\delta,\delta)$ for the infinity norm. Furthermore, this function has values in $[0,1]$.

Let $\sigma$ be a ReLU-like activation with asymptote $b\cdot x+c$ in $+\infty$. For any $\delta>0$ and error tolerance $\epsilon>0$, there exists a scaling $M>0$ such that $x\mapsto \frac{\sigma(Mx)}{Mb}$ approximates the ReLU function within $\epsilon$ for the infinity norm.
\end{restatable}

\begin{proof}
See Appendix \ref{appendix sec: NN upper bound sec}.
\end{proof}
\noindent With this Lemma we can now prove Theorem \ref{thm: shallow NN bound generalized activation.}.
\begin{proof}[Proof of Theorem \ref{thm: shallow NN bound generalized activation.}]
We start by the class of ReLU-like activations. Let $\sigma$ be a ReLU-like activation function. Lemma \ref{lemma:approx activations} shows that one can approximate arbitrarily well the ReLU activation with a linear map $\sigma$. Take the neural network approximator $\hat f$ of a target function $f$ given by Theorem \ref{thm: shallow NN bound.}. At each node, we can add the linear map corresponding to $x\mapsto \frac{\sigma(Mx)}{Mb}$ with no additional neuron nor parameter. Because the approximation is continuous, we can take $M>0$ arbitrarily large in order to approximate $\hat f$ with arbitrary precision on the compact $[0,1]^d$.

The same argument holds for sigmoid-like activation functions in order to reduce the problem to Heaviside activation functions. Although quadratic approximations for univariate functions similar to Lemma \ref{lemma:quadratic approximation 1D} are not valid for general sigmoid-like activations -- in particular the Heaviside --- we can obtain an analog to Lemma \ref{lemma:1D function approx.} as Lemma \ref{lemma:1D function approx sigmoid} given in the Appendix. This results is an increased number of neurons. In order to approximate a target function $f\in X^{2,\infty}(\Omega)$, we use the same structure as the neural network constructed for ReLU activations and use the same notations as in the proof of Theorem \ref{thm: shallow NN bound.}. The first difference lies in the approximation of $\log \phi_{l_j,i_j}$ in the first layer. Instead of using Corollary \ref{corollary:quadratic approximation log}, we use Lemma \ref{lemma:1D function approx sigmoid}. Therefore, $\frac{12d}{\tilde \epsilon}\log \frac{3}{\tilde \epsilon}$ neurons are needed to compute a $\tilde\epsilon/(3d)-$approximation of $\max(\log \phi_{l_j,i_j},\log (\tilde\epsilon/3))$. The second difference is in the approximation of the exponential in the second layer. Again, we use Lemma \ref{lemma:1D function approx sigmoid} to construct a $\tilde \epsilon/3-$approximation of the exponential on $\R_-$ with $\frac{6}{\tilde\epsilon}$ neurons for the second layer. As a result, the first layer contains at most $2^{n+2}\frac{3d^2}{\tilde \epsilon}\log \frac{1}{\tilde \epsilon}$ neurons for $\epsilon$ sufficiently small, and the second layer contains $\left| U_n^{(1)}\right|\frac{6}{\tilde\epsilon}$ neurons. Using the same estimates as in the proof of Theorem \ref{thm: shallow NN bound.} shows that the constructed neural network has at most $N_1$ (resp. $N_2$) neurons on the first (resp. second) layer where
\begin{align*}
    N_1 &\underset{\epsilon\to0}{\sim}   \frac{3\cdot 2^5\cdot d^{5/2}}{8^{\frac{d}{2}}(2\log 2)^{\frac{d-1}{2}} d!^{\frac{1}{2}}}\frac{ |f|_{\bold 2,\infty}^{3/2}}{\epsilon^{3/2}} \left(\log \frac{1}{\epsilon} \right)^{\frac{d+1}{2}},\\
    N_2 &\underset{\epsilon\to0}{\sim} \frac{24\cdot d^{\frac{3}{2}}}{8^{\frac{d}{2}}(2\log 2)^{\frac{3(d-1)}{2}} d!^{\frac{3}{2}}}\cdot\frac{|f|_{\bold 2,\infty}^{3/2}}{\epsilon^{3/2}} \left(\log \frac{1}{\epsilon} \right)^{\frac{3(d-1)}{2}}.
\end{align*}
This ends the proof.
\end{proof}

\section{The Representation Power of Deep Neural Networks}\label{sec: deep nn.}

Bungartz and Griebel's \cite{bungartz2004sparse} parametrization via sparse grids approximates the unit ball of the Korobov space within $\epsilon$ with $O(\epsilon^{-\frac{1}{2}} (\log\frac{1}{\epsilon})^{\frac{3(d-1)}{2}})$ parameters (Theorem \ref{thm: sparse grid approx.}). Montanelli and Du \cite{montanelli2019new} used the sparse grid approximation to construct deep neural networks, with ReLU activation that can approximate Korobov functions with $O(\epsilon^{-\frac{1}{2}} (\log\frac{1}{\epsilon})^{\frac{3(d-1)}{2}+1})$ neurons, and depth $O(\log\frac{1}{\epsilon})$ for the $L^{\infty}$ norm.

% \begin{theorem}[Montanelli and Du. \cite{montanelli2019new}]\label{thm: Deep net bound.}

% \end{theorem}

We show that we can improve Montanelli and Du. \cite{montanelli2019new}'s bound for deep neural networks with $\mathcal C^2$ non-linear activation functions. We prove that we only need $O(\epsilon^{-\frac{1}{2}}(\log \frac{1}{\epsilon})^{\frac{3(d-1)}{2}})$ neurons and a fixed depth, independent of $\epsilon$, to approximate the unit ball of the Korobov space within $\epsilon$ in the $L^{\infty}$ norm.

\begin{theorem}
\label{thm: fixed deep upper bound}
Let $\sigma\in\mathcal C^2$ be a non-linear activation function. Let $\epsilon>0$. For any function $f\in X^{2,\infty}(\Omega)$, there exists a neural network of depth $\lceil\log_2 d\rceil + 1$, with ReLU activation on the first layer and activation function $\sigma$ for the next layers, $O(\epsilon^{-\frac{1}{2}}(\log \frac{1}{\epsilon})^{\frac{3(d-1)}{2}})$ neurons, and $O(\epsilon^{-\frac{1}{2}}(\log \frac{1}{\epsilon})^{\frac{3(d-1)}{2}})$ training parameters approximating $f$ within $\epsilon$ for the infinity norm.
\end{theorem}

% The depth of the constructed network in the proof of Theorem \ref{thm: fixed deep upper bound} is $\lceil\log_2 d\rceil + 1$ (asymptotic in $d$), independent of the approximation error $\epsilon$. 
This new bound on the number of parameters and neurons matches the approximation power of sparse grids. In fact, sparse grids use $\Theta(\epsilon^{-\frac{1}{2}}(\log \frac{1}{\epsilon})^{\frac{3(d-1)}{2}})$ parameters (weights of basis functions) to approximate Korobov functions within $\epsilon$ in infinity norm. Our construction in the proof of Theorem \ref{thm: fixed deep upper bound} shows that deep neural networks with fixed depth in $\epsilon$ can fully encode sparse grids approximators. Neural networks are therefore more powerful function approximators. In particular, any sparse grid approximation using $O(N(\epsilon))$ parameters, can be represented exactly by a neural network using $O(N(\epsilon))$ neurons.

\begin{figure}[h!]
    \centering
    \begin{tikzpicture}
    [   cnode/.style={draw=black,fill=#1,minimum width=3mm,circle},
    ]
    %input neurons
	\node[cnode=blue!40, label=above:$x_1$] (x1) at (-2,3) {};    	
	\node[cnode=blue!40, label=above:$x_2$] (x2) at (-1,3) {}; 
	\draw (0.5,3) node {$\cdots$};
	\node[cnode=blue!40, label=above:$x_d$] (xd) at (2,3) {}; 
	
	%1st layer neurons
	\foreach \x in {-6,-3,-1,2,5}{
	    \foreach \k in {0,1,2,3}
	    \node[draw,circle,minimum size = 2mm, inner sep=-2] (1-\x\k) at (\x+\k*1/3,2) {} ;
	}
	\foreach \x in {-4,1,4}
	\draw (\x,2) node {$\cdots$};

	%edges from input to 1st layer
	\foreach \x in {-6,-3}{
	    \foreach \k in {0,1,2,3}
	    \draw        [gray!80] (x1) -- (1-\x\k);
	}
	\foreach \x in {-1,2}{
	    \foreach \k in {0,1,2,3}
	    \draw        [gray!80] (x2) -- (1-\x\k);
	}
	\foreach \x in {5}{
	    \foreach \k in {0,1,2,3}
	    \draw        [gray!80] (xd) -- (1-\x\k);
	}
	
	%2nd layer neurons
	\node[rectangle, minimum height=3cm, minimum width=1cm, fill = gray!40] (R1) at (-7,-1) {$\mathcal P_{\mb 1,\mb 1}$};
	\draw node[minimum height=4mm,minimum width=1cm] (R1node) at (-7,  0.3) {};
	\draw node[minimum height=0mm, inner sep=-2] (R1below) at (-7,  -2.5) {};
	\node[rectangle, minimum height=3cm, minimum width=1cm, fill = gray!40] (R2) at (-5,-1) {};
	\draw node[minimum height=4mm,minimum width=1cm] (R2node) at (-5,  0.3)  {} ;
	\draw node[minimum height=0mm, inner sep=-2] (R2below) at (-5,    -2.5) {};
	\node[rectangle, minimum height=3cm, minimum width=1cm, fill = gray!40] (R3) at (-3,-1) {};
	\draw node[minimum height=4mm,minimum width=1cm] (R3node) at (-3,  0.3)  {} ;
	\draw node[minimum height=0mm, inner sep=-2] (R3below) at (-3,    -2.5) {};
	
	\draw (0,-1) node {$\cdots$};
	
	\node[rectangle, minimum height=3cm, minimum width=1cm, fill = gray!40] (R4) at (2,-1) {$\mathcal P_{\mb l,\mb i}$};
	\draw node[minimum height=4mm,minimum width=1cm] (R4node) at (2,  0.3) {}  ;
	\draw node[minimum height=0mm, inner sep=-2] (R4below) at (2,    -2.5) {};
	
	\draw (4.5,-1) node {$\cdots$};
	
	\node[rectangle, minimum height=3cm, minimum width=1cm, fill = gray!40] (R5) at (7,-1) {};
	\draw node[minimum height=4mm,minimum width=1cm] (R5node) at (7,  0.3)  {} ;
	\draw node[minimum height=0mm, inner sep=-2] (R5below) at (7,    -2.5) {};
	
	%edges from 1st to 2nd layer
	\foreach \x in {-6,-1}{
	    \foreach \k in {0,1,2,3}
	    \draw        [gray!80] (1-\x\k) to (R1node);
	}
	
	\foreach \x in {-6,-1,5}{
	    \foreach \k in {0,1,2,3}
	    \draw        [gray!80] (1-\x\k) to (R5node);
	}
	%output node
	\node[cnode=blue!40, label=below:$\displaystyle \widetilde{
        f}(\mb x)$] (output) at (0,-4) {}; 
	
	%edges to output
	\foreach \x in {1,2,3,4,5}
	\draw        [gray!80] (R\x below) to (output);
	%explanations
	\foreach \x in {-6,-3,-1,2,5}
	\draw [dashed] (\x+0.5,2) ellipse (0.8cm and 0.3cm);
	\draw (-5.5,  1.4) node {$\phi_{1,1}(x_1)$};
	\draw (-2.5,  1.4) node {$   \phi_{n,2^n-1}(x_1)$};
	
	\draw (-0.5,  1.4) node {$   \phi_{1,1}(x_2)$};
	\draw (2.5,  1.4) node {$   \phi_{n,2^n-1}(x_2)$};
	
	\draw (5.5,  1.4) node {$   \phi_{n,2^n-1}(x_d)$};
	
	%description of basis functions
	\draw (3,-3) node {$\displaystyle\phi_{\mb l,\mb i}(\mb x)= \prod_{j=1}^d \phi_{l_j,i_j}(x_j)$};
	
	%description of 1st layer
	\draw [decorate,decoration={brace,amplitude=5pt,raise=4ex}]
  (-5.5,3.5) -- (5.5,3.5) node[midway,yshift=3em]{$ O(\epsilon^{-1/2}(\log\frac{1}{\epsilon})^{\frac{d-1}{2}})$ functions $\phi_{l_j,i_j}(x_j)$};
	
	%number of basis functions
    \draw [decorate,decoration={brace,amplitude=5pt,mirror,raise=4ex}]
  (-7,-4.5) -- (7,-4.5) node[midway,yshift=-3em]{$\left|U_n^{(1)}\right| = O(\epsilon^{-1/2}(\log\frac{1}{\epsilon})^{\frac{3(d-1)}{2}})$ basis functions $\phi_{\mb l,\mb i}$};
	
	%depth
	\draw [<->,>=latex] (-2.2,-2.5) -- node[right]{$\lceil \log_2 d\rceil$} (-2.2,0.5);
    \end{tikzpicture}
    \caption{Deep neural network approximating a Korobov function $f\in X^{2,\infty}(\Omega)$ within $\epsilon$ in infinity norm. It contains $|U_n^{(1)}|$ sub-networks $\mathcal P_{\mb l,\mb i}$ implementing the basis the product $\phi_{\mb l,\mb i}$ with arbitrary precision. The complete network has $O(\epsilon^{-1/2}(\log\frac{1}{\epsilon})^{\frac{3(d-1)}{2}})$ neurons and depth $\lceil \log_2 d\rceil+1$.}
    \label{fig:deep nn}
\end{figure}
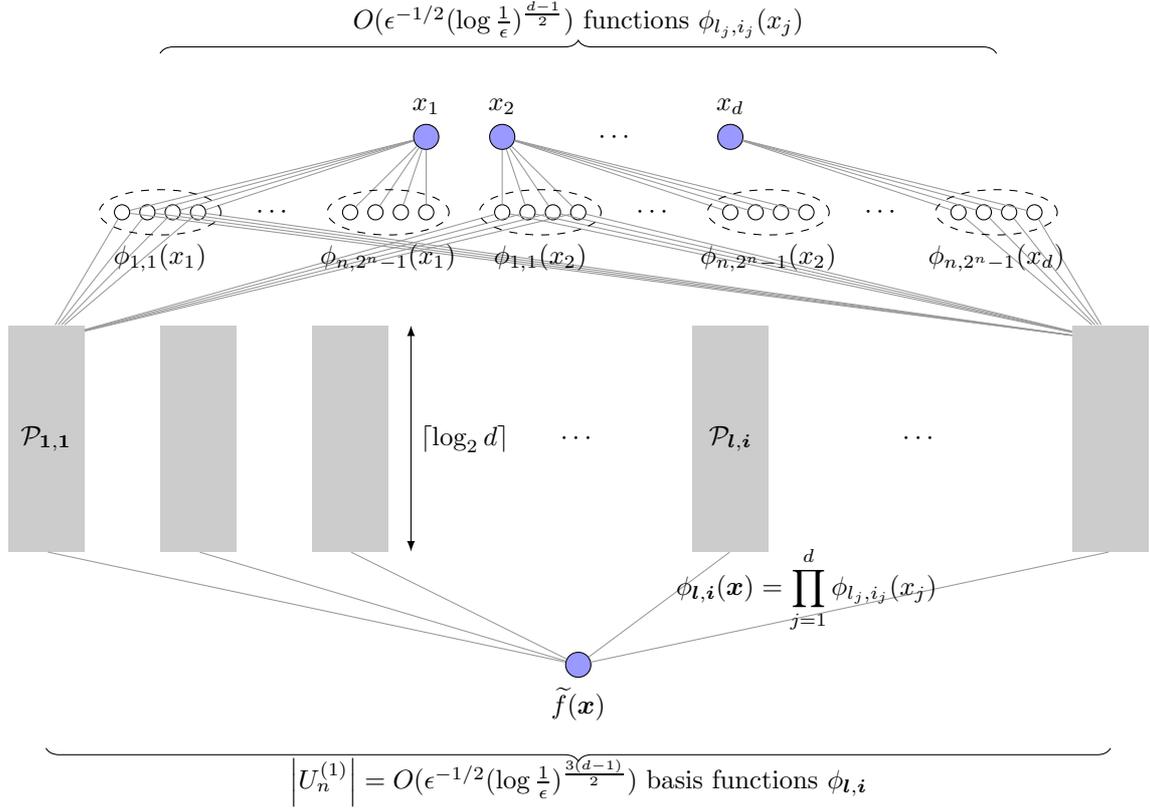

As in the construction of our shallow network in the previous section, we will make use of a network that approximates the product function, deep in this case. The following result of Lin, Tegmark and Rolnick \cite[Appendix A]{lin2017does} shows that deep neural networks can represent exactly the product function.

\begin{proposition}[Lin, Tegmark and Rolnick \cite{lin2017does}]
\label{prop:exact product deep nn}
Let $\sigma$ be $\mathcal C^2$ non linear activation function. For any approximation error $\epsilon>0$, there exists a neural network with $\lceil \log_2 d\rceil$ hidden layers and activation $\sigma$, using at most $8d$ neurons arranged in a binary tree network that approximates the product function $\prod_{i=1}^d x_i$ on $[0,1]^d$ within $\epsilon$ for the infinity norm.
\end{proposition}
An important remark is that the structure of the constructed neural network is independent of $\epsilon$. In particular, the depth and number of neurons is independent of the approximation precision $\epsilon$, which we refer as \emph{exact} approximation. It is known that an exponential number of neurons is needed in order to \emph{exactly} approximate the product function with a single-layer neural network \cite{lin2017does}, however, the question of whether one could approximate the product with a shallow network and a polynomial number of neurons, remained open. In Proposition \ref{prop:product_approx}, we answer positively to this question
by constructing an $\epsilon-$approximating neural network of depth $2$ with ReLU activation and $O(d^{3/2}\epsilon^{-1/2}\log\frac{1}{\epsilon})$ neurons. Using the same ideas as in Subsection \ref{subsection: generalizing activations}, we can generalize this result to other activation functions and obtain an $\epsilon-$approximating neural network of depth $2$ with $O(d^{3/2}\epsilon^{-1/2}\log\frac{1}{\epsilon})$ neurons for a ReLU-like activation, or $O(d^2\epsilon^{-1}\log\frac{1}{\epsilon})$ neurons for a sigmoid-like activation. 
% Combining this product approximation with sparse grids, we now prove Theorem \ref{thm: fixed deep upper bound}.

\begin{proof}[Proof of Theorem \ref{thm: fixed deep upper bound}]
Let $\epsilon>0$. We construct a similar structure to the network defined in Theorem \ref{thm: shallow NN bound.} by using the sparse grid approximation of Subsection \ref{subsection: sparse grids}. For a given $n$, let $f_n^{(1)}$ be the projection of $f$ in the approximation space $V_n^{(1)}$ (defined in \eqref{eq: V1.}) and $U_n^{(1)}$ (defined in \eqref{eq: Un1.}) the set of indices $(\mb l, \mb i)$ of basis functions present in $V_n^{(1)}$. Recall $f_n^{(1)}$ can be uniquely decomposed as 
\begin{equation*}
    f_n^{(1)}(\bold x)=\sum_{(\bold l, \bold i)\in U_n^{(1)}} v_{\bold l,\bold i} \phi_{\bold l,\bold i}(\bold x).
\end{equation*}
where $\phi_{\mb l,\mb i} = \prod_{j=1}^d \phi_{l_j,i_j}$ are the basis functions defined in \eqref{eq: hat functions lj ij}.
In the first layer, we compute exactly the piece-wise linear hat functions $\phi_{l_j,i_j}$, then in the next set of layers, we use the product-approximating neural network given by Proposition \ref{prop:exact product deep nn} to compute the basis functions $\phi_{\mb l,\mb i} = \prod_{j=1}^d \phi_{l_j,i_j}$ (see Figure \ref{fig:deep nn}). The output layer computes the weighted sum
$\sum_{(\bold l, \bold i)\in U_n^{(1)}} v_{\bold l,\bold i} \phi_{\bold l,\bold i}(\bold x)$ and outputs $f_n^{(1)}$. Because the approximation has arbitrary precision, we can chose the network of Proposition \ref{prop:exact product deep nn} such that the resulting network $\hat{f}$ verifies $\|\hat{f} - f^{(1)}_n \|_{\infty} \leq \epsilon/2$.

More precisely, as $\phi_{l_j,i_j}$ is piece-wise linear with four pieces, we can compute it exactly with four neurons with ReLU activation on a single layer (Lemma \ref{lemma:piece-wise affine}). Our first layer is composed of the union of all these ReLU neurons, for the $d(2^n-1)$ indices $l_j,i_j$ such that $1\leq j\leq d$, $1\leq l_j\leq n$, $1\leq i_j\leq 2^{l_j}$ and $i_j$ is odd. Therefore, it contains at most $d2^{n+2}$ neurons with ReLU activation. The second set of layers is composed of the union of product-approximating neural networks to compute $\phi_{\mb l,\mb i}$ for all $(\mb l,\mb i)\in U_n^{(1)}.$ This set of layers contains $\lceil \log_2 d\rceil$ layers with activation $\sigma$ and at most $|U_n^{(1)}|\cdot 8d$ neurons. The output of these two sets of layers is an approximation of the basis functions $\phi_{\mb l,\mb i}$ with arbitrary precision. Consequently, the final output of the complete neural network is an approximation of $f_n^{(1)}$ with arbitrary precision. Similarly to the proof of Theorem \ref{thm: shallow NN bound.}, we can chose the smallest $n$ such that $\|f-f_n^{(1)}\|_\infty\leq \epsilon/2$ (see \eqref{eq: n epsilon} for details). Finally, the network has depth at most $\log_2 d +2$ and $N$ neurons where
\begin{equation*}
    N = 8d |U_{n_{\epsilon}}^{(1)}| \underset{\epsilon\to0}{\sim} \frac{ 2^5\cdot d^{5/2}}{8^{\frac{d}{2}}(2\log 2)^{\frac{3(d-1)}{2}} d!^{\frac{3}{2}}}\cdot\sqrt{\frac{|f|_{\bold 2,\infty}}{\epsilon}} \left(\log \frac{1}{\epsilon} \right)^{\frac{3(d-1)}{2}}.
\end{equation*}
The parameters of the network depending on the function are exactly the coefficients $v_{\mb l,\mb i}$ of the sparse grid approximation. Hence, the network has $O(\epsilon^{-\frac{1}{2}}(\log \frac{1}{\epsilon})^{\frac{3(d-1)}{2}})$ training parameters.
\end{proof}

\section{Neural Networks are Near-Optimal Function Approximators}\label{sec: NN are optimal.}
In the previous sections, we proved and discussed upper bounds on the number of neurons and training parameters needed by deep and shallow neural networks to approximate the Korobov space $X^{2,\infty}(\Omega)$. In this section, we investigate how good is the performance of neural networks as functional approximators. We prove a lower bound on the number of parameters needed by any continuous function approximator to approximate the Korobov space. We will show in particular that neural networks, deep and shallow, nearly match this lower bound, making them near-optimal function approximators. Let us first formalize the notion of continuous function approximators, following the framework of DeVore et al. \cite{devore1989optimal}.

For any Banach space $\mathcal X$ --- one can think of $\mathcal X$ as a space of functions --- and a subset $K\subset \mathcal X$ of elements to approximate, we define a continuous function approximator with $N$ parameters as a continuous parametrization $a:K\to \R^N$ together with a reconstruction scheme which is a $N-$dimensional manifold $\mathcal M_N:\R^N\to \mathcal{X}$. For any element $f\in K$, the approximation given by the class is $\mathcal M_N(a(f))$ : the parametrization $a$ is derived continuously from the function $f$ and then given as input to the reconstruction manifold that outputs an approximation function in $\mathcal{X}$. The error of this function approximator is defined as
\begin{equation*}
    E_{N,a,\mathcal M_N}(K)_{\mathcal{X}}:=\sup_{f\in K} |f-\mathcal M_N(a(f))|_K.
\end{equation*}
The best function approximator for the space $K$ minimizes this error. The minimal error for space $K$ is given by
\begin{equation*}
    E_N(K)_{\mathcal{X}} = \min_{a,\mathcal M_N} E_{N,a,\mathcal M_N}(K). 
\end{equation*}
In other terms, a continuous function approximator with $N$ parameters cannot hope to approximate $K$ better than within $E_N(K)_{\mathcal{X}}$. A  class of function approximators is a set of function approximators with a given structure. For example, neural networks with continuous parametrizations are a class of function approximators where the number of parameters corresponds to the number of training parameters i.e, the parameters of the network that depend on the desired function to approximate. We say that a class of function approximators is optimal for the space of functions $K$ if it matches this minimal error asymptotically in $N$, within a constant multiplicative factor. In other words, the number of parameters needed by the class to approximate functions in $K$ within $\epsilon$ matches asymptotically, within a constant, the least number of parameters $N$ needed to satisfy $E_N(K)_{\mathcal{X}} \leq \epsilon$. The norm considered in the approximation of the functions of $K$ is the norm associated to the space $\mathcal{X}$.

DeVore et al. \cite{devore1989optimal} showed that this minimum error $E_N(K)_{\mathcal{X}}$ admits a lower bound characterized by the Bernstein width of the subset $K\subset \mathcal X$ defined as
\begin{equation}\label{eq: Bernstein width}
    b_N(K)_{\mathcal X}:=\sup_{X_{N+1}} \sup\{\rho: \; \rho U(X_{N+1})\subset K\},
\end{equation}
where the outer sup is taken over all $N+1$ dimensional linear sub-spaces of $\mathcal{X}$, and $U(Y)$ denotes the unit ball of $Y$ for any linear subspace $Y$ of $\mathcal X$.

\begin{theorem}[DeVore et al. \cite{devore1989optimal}]
Let $\mathcal{X}$ be a normed linear space and let $K \subset \mathcal{X}$. The following inequality holds
\begin{equation*}
    E_N(K)_{\mathcal{X}}\geq b_N(K)_{\mathcal X}.
\end{equation*}
\end{theorem}

\paragraph{Approximating Sobolev functions.} DeVore et al. derived a lower bound on the Bernstein width of the Sobolev space to show that the Sobolev space has an inherent curse of dimensionality.

\begin{theorem}[DeVore et al. \cite{devore1989optimal}] \label{theorem: Bernstein lower bound}
Let $p,m\geq 1$. Take $\mathcal X=L^p([0,1]^d)$ and $K = \{f\in W^{r,p} \; : \; |f|_{W^{r,p}}\leq 1\}$ the unit ball of the Sobolev space. Then, there exists $c>0$ such that
\begin{equation*}
    E_N(K)_{\mathcal{X}} \geq c N^{-\frac{r}{d}}.
\end{equation*}
Equivalently, for all $\epsilon>0$, a continuous function approximator approximating the unit Sobolev ball within $\epsilon$ in $L^p$ norm uses at least $\Theta(\epsilon^{-\frac{d}{r}})$ parameters.
\end{theorem}

This lower bounds matches Mhaskar's upper's bound on the number of neurons needed for a neural network to approximate the Sobolev space (Theorem \ref{thm: Mhaskar upper bound}). Mhaskar used one layer neural networks to construct the approximation realizing the bound of Theorem \ref{thm: Mhaskar upper bound}, therefore, the number of training parameter these neural networks is also $\Theta(\epsilon^{-\frac{d}{r}})$, matching DeVore's lower bound. Furthermore, Mhaskar's construction is continuous. Hence, even if neural networks suffer from the curse of dimensionality, they are optimal function approximators of the Sobolev space. 

\paragraph{Approximating Korobov functions.}

We turn now to the space of interest, the Korobov space. We prove a lower bound on the least number of parameters any class of continuous function approximators needs to approximate functions of the Korobov space. 

\begin{theorem}
\label{thm:lower bound korobov}
Take $\mathcal X=L^\infty(\Omega)$ and $K=\{f\in X^{2,\infty}(\Omega) \; : \; |f|_{X^{2,\infty}(\Omega)}\leq 1\}$ the unit ball of the Korobov space. Then, there exists $c>0$ such that
\begin{equation*}
    E_N(K)_{\mathcal{X}} \geq \frac{c}{N^2}(\log N)^{d-1}.
\end{equation*}
Equivalently, for $\epsilon>0$, a continuous function approximator approximating the unit Korobov ball within $\epsilon$ in $L^\infty$ norm uses at least $\Theta(\epsilon^{-\frac{1}{2}}( \log\frac{1}{\epsilon})^{\frac{d-1}{2}})$ parameters.
\end{theorem}

This lower bound nearly matches the upper bound on the number of training parameters needed by deep and shallow neural networks to approximate the Korobov space within $\epsilon$: $O(\epsilon^{-\frac{1}{2}}(\log \frac{1}{\epsilon})^{\frac{3(d-1)}{2}})$ (Theorem \ref{thm: shallow NN bound.} and Theorem \ref{thm: fixed deep upper bound}). It exhibits the same exponential dependence in $d$ with base $\log \frac{1}{\epsilon}$ and the same main dependence on $\epsilon$ of $\epsilon^{-\frac{1}{2}}$. Note that the upper and lower bound can be rewritten as $O((\frac{1}{\epsilon})^{1/2+\delta})$ for all $\delta>0$. Moreover, our constructions in Theorem \ref{thm: shallow NN bound.} and Theorem \ref{thm: fixed deep upper bound} are continuous, which comes directly from the continuity of the sparse grid parameters (see bound on $v_{\mb l,\mb i}$ in Theorem \ref{thm: sparse grid approx.}). Our bounds prove therefore that deep and shallow neural networks are near optimal classes of function approximators for the Korobov space.

% \subsection{Proof of Theorem \ref{thm:lower bound korobov}}
\begin{proof}[Proof of Theorem \ref{thm:lower bound korobov}]
Our goal is to define an appropriate subspace $X_{N+1}$ in order to get a good lower bound on the Bernstein width $b_N(K)_{\mathcal X}$, defined in \eqref{eq: Bernstein width}, which in turn provides a lower bound on the approximation error (Theorem \ref{theorem: Bernstein lower bound}).

To do so, we introduce the Deslaurier-Dubuc interpolet $\phi^{(L)}:\R\to\R$. The construction of this function uses an interpolating scheme on binary rationals. First, $\phi^{(L)}$ is defined on all integers $\phi^{(L)}(k)=\mb{1}_{k=0}$ for $k\in \Z$. Then, we define the function on half integers $\frac{2k+1}{2}$ by fitting a polynomial of degree $2L-1$ interpolating the hat function $\phi$ (defined in \eqref{eq: hat function}) at $k-L+1, \cdots, k+L$. Iteratively, we define the interpolet on binary rationals of the form $\frac{2k+1}{2^{j+1}}$ from the value of the interpolet on rationals with denominator $2^j$. Specifically, let $P_{j,k}^{(L)}$ be the unique polynomial of degree $2L$ interpolating $\phi^{(L)}$ at points $\frac{k'}{2^j}$ for $k'\in\{k-L+1,\cdots, k+L\}$. We set
\begin{equation*}
    \phi^{(L)}\left(\frac{2k+1}{2^{j+1}}\right) := P_{j,k}^{(L)} \left(\frac{2k+1}{2^j}\right).
\end{equation*}
For example, for $L=2$ we get $
    \phi^{(2)}(\frac{2k+1}{2^{j+1}}) := \frac{9}{16}\phi^{(2)} (\frac{k}{2^j})+\frac{9}{16}\phi^{(2)} (\frac{k+1}{2^j}) - \frac{1}{16}\phi^{(2)} (\frac{k-2}{2^j})-\frac{1}{16}\phi^{(2)} (\frac{k+3}{2^j})$.
This process defines the interpolet on binary rationals. We then extend the function to the real line by continuity. See Figure \ref{fig:interpolets} for an illustration. Deslaurier and Dubuc \cite{deslauriers1989symmetric} proved that the regularity of the interpolet is an increasing function of $L$ the degree of interpolation. We now prove results in the case $L=2$.

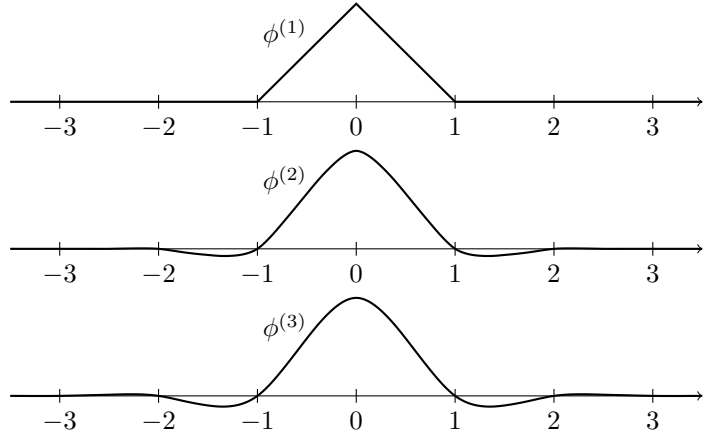
\begin{figure}
    \centering
    \begin{tikzpicture} [scale = 1.3]
    \foreach \y in {0,-1.5,-3}
    \draw[,->] (-3.5,\y) -- (3.5,\y);
    \foreach \x in {-3,-2,-1,0,1,2,3}
    \draw (\x,2pt) -- (\x, -2pt) node[anchor=north] {$\x$};
    
    \foreach \x in {-3,-2,-1,0,1,2,3}
    \draw (\x,2pt-1.5cm) -- (\x, -2pt-1.5 cm) node[anchor=north] {$\x$};
    
    \foreach \x in {-3,-2,-1,0,1,2,3}
    \draw (\x,2pt-3cm) -- (\x, -2pt-3cm) node[anchor=north] {$\x$};
    
    \draw [thick] (-3.5,0)-- (-1,0) -- (0,1) -- (1,0)--(3.5,0);

    \draw  [thick](-3.5,-1.5)-- (-3.0,-1.5) -- (-2.96875,-1.5) -- (-2.9375,-1.5) -- (-2.90625,-1.5000009536743164) -- (-2.875,-1.5) -- (-2.84375,-1.4999914169311523) -- (-2.8125,-1.4999847412109375) -- (-2.78125,-1.4999828338623047) -- (-2.75,-1.5) -- (-2.71875,-1.5000629425048828) -- (-2.6875,-1.5001373291015625) -- (-2.65625,-1.500197410583496) -- (-2.625,-1.500244140625) -- (-2.59375,-1.5002832412719727) -- (-2.5625,-1.500274658203125) -- (-2.53125,-1.5002021789550781) -- (-2.5,-1.5) -- (-2.46875,-1.4995536804199219) -- (-2.4375,-1.498992919921875) -- (-2.40625,-1.4983949661254883) -- (-2.375,-1.497802734375) -- (-2.34375,-1.4972944259643555) -- (-2.3125,-1.4968414306640625) -- (-2.28125,-1.4964466094970703) -- (-2.25,-1.49609375) -- (-2.21875,-1.4957256317138672) -- (-2.1875,-1.4954681396484375) -- (-2.15625,-1.4954252243041992) -- (-2.125,-1.49560546875) -- (-2.09375,-1.4959917068481445) -- (-2.0625,-1.49676513671875) -- (-2.03125,-1.4980087280273438) -- (-2.0,-1.5) -- (-1.96875,-1.5032119750976562) -- (-1.9375,-1.50714111328125) -- (-1.90625,-1.5114755630493164) -- (-1.875,-1.51611328125) -- (-1.84375,-1.5208654403686523) -- (-1.8125,-1.5256805419921875) -- (-1.78125,-1.530508041381836) -- (-1.75,-1.53515625) -- (-1.71875,-1.5393619537353516) -- (-1.6875,-1.5432891845703125) -- (-1.65625,-1.5470266342163086) -- (-1.625,-1.550537109375) -- (-1.59375,-1.5537958145141602) -- (-1.5625,-1.556854248046875) -- (-1.53125,-1.5597038269042969) -- (-1.5,-1.5625) -- (-1.46875,-1.5655403137207031) -- (-1.4375,-1.568389892578125) -- (-1.40625,-1.570775032043457) -- (-1.375,-1.572509765625) -- (-1.34375,-1.5732908248901367) -- (-1.3125,-1.5731964111328125) -- (-1.28125,-1.572183609008789) -- (-1.25,-1.5703125) -- (-1.21875,-1.5678157806396484) -- (-1.1875,-1.5641326904296875) -- (-1.15625,-1.558802604675293) -- (-1.125,-1.5517578125) -- (-1.09375,-1.5430269241333008) -- (-1.0625,-1.5318603515625) -- (-1.03125,-1.5178985595703125) -- (-1.0,-1.5) -- (-0.96875,-1.4762420654296875) -- (-0.9375,-1.4486083984375) -- (-0.90625,-1.4183063507080078) -- (-0.875,-1.3857421875) -- (-0.84375,-1.3516979217529297) -- (-0.8125,-1.316253662109375) -- (-0.78125,-1.2795524597167969) -- (-0.75,-1.2421875) -- (-0.71875,-1.2049140930175781) -- (-0.6875,-1.167388916015625) -- (-0.65625,-1.1294841766357422) -- (-0.625,-1.09130859375) -- (-0.59375,-1.0528736114501953) -- (-0.5625,-1.01434326171875) -- (-0.53125,-0.9758834838867188) -- (-0.5,-0.9375) -- (-0.46875,-0.8991165161132812) -- (-0.4375,-0.86114501953125) -- (-0.40625,-0.8238773345947266) -- (-0.375,-0.78759765625) -- (-0.34375,-0.7527065277099609) -- (-0.3125,-0.719146728515625) -- (-0.28125,-0.6870307922363281) -- (-0.25,-0.65625) -- (-0.21875,-0.6264152526855469) -- (-0.1875,-0.598480224609375) -- (-0.15625,-0.5732173919677734) -- (-0.125,-0.55078125) -- (-0.09375,-0.5311985015869141) -- (-0.0625,-0.515625) -- (-0.03125,-0.504638671875) -- (0.0,-0.5) -- (0.03125,-0.504638671875) -- (0.0625,-0.515625) -- (0.09375,-0.5311985015869141) -- (0.125,-0.55078125) -- (0.15625,-0.5732173919677734) -- (0.1875,-0.598480224609375) -- (0.21875,-0.6264152526855469) -- (0.25,-0.65625) -- (0.28125,-0.6870307922363281) -- (0.3125,-0.719146728515625) -- (0.34375,-0.7527065277099609) -- (0.375,-0.78759765625) -- (0.40625,-0.8238773345947266) -- (0.4375,-0.86114501953125) -- (0.46875,-0.8991165161132812) -- (0.5,-0.9375) -- (0.53125,-0.9758834838867188) -- (0.5625,-1.01434326171875) -- (0.59375,-1.0528736114501953) -- (0.625,-1.09130859375) -- (0.65625,-1.1294841766357422) -- (0.6875,-1.167388916015625) -- (0.71875,-1.2049140930175781) -- (0.75,-1.2421875) -- (0.78125,-1.2795524597167969) -- (0.8125,-1.316253662109375) -- (0.84375,-1.3516979217529297) -- (0.875,-1.3857421875) -- (0.90625,-1.4183063507080078) -- (0.9375,-1.4486083984375) -- (0.96875,-1.4762420654296875) -- (1.0,-1.5) -- (1.03125,-1.5178985595703125) -- (1.0625,-1.5318603515625) -- (1.09375,-1.5430269241333008) -- (1.125,-1.5517578125) -- (1.15625,-1.558802604675293) -- (1.1875,-1.5641326904296875) -- (1.21875,-1.5678157806396484) -- (1.25,-1.5703125) -- (1.28125,-1.572183609008789) -- (1.3125,-1.5731964111328125) -- (1.34375,-1.5732908248901367) -- (1.375,-1.572509765625) -- (1.40625,-1.570775032043457) -- (1.4375,-1.568389892578125) -- (1.46875,-1.5655403137207031) -- (1.5,-1.5625) -- (1.53125,-1.5597038269042969) -- (1.5625,-1.556854248046875) -- (1.59375,-1.5537958145141602) -- (1.625,-1.550537109375) -- (1.65625,-1.5470266342163086) -- (1.6875,-1.5432891845703125) -- (1.71875,-1.5393619537353516) -- (1.75,-1.53515625) -- (1.78125,-1.530508041381836) -- (1.8125,-1.5256805419921875) -- (1.84375,-1.5208654403686523) -- (1.875,-1.51611328125) -- (1.90625,-1.5114755630493164) -- (1.9375,-1.50714111328125) -- (1.96875,-1.5032119750976562) -- (2.0,-1.5) -- (2.03125,-1.4980087280273438) -- (2.0625,-1.49676513671875) -- (2.09375,-1.4959917068481445) -- (2.125,-1.49560546875) -- (2.15625,-1.4954252243041992) -- (2.1875,-1.4954681396484375) -- (2.21875,-1.4957256317138672) -- (2.25,-1.49609375) -- (2.28125,-1.4964466094970703) -- (2.3125,-1.4968414306640625) -- (2.34375,-1.4972944259643555) -- (2.375,-1.497802734375) -- (2.40625,-1.4983949661254883) -- (2.4375,-1.498992919921875) -- (2.46875,-1.4995536804199219) -- (2.5,-1.5) -- (2.53125,-1.5002021789550781) -- (2.5625,-1.500274658203125) -- (2.59375,-1.5002832412719727) -- (2.625,-1.500244140625) -- (2.65625,-1.500197410583496) -- (2.6875,-1.5001373291015625) -- (2.71875,-1.5000629425048828) -- (2.75,-1.5) -- (2.78125,-1.4999828338623047) -- (2.8125,-1.4999847412109375) -- (2.84375,-1.4999914169311523) -- (2.875,-1.5) -- (2.90625,-1.5000009536743164) -- (2.9375,-1.5) -- (2.96875,-1.5) -- (3.0,-1.5) -- (3.5,-1.5);

    \draw [thick] (-3.5,-3.0) -- (-3.46875,-3.000126521508723) -- (-3.4375,-3.0002796533496485) -- (-3.40625,-3.000446856397567) -- (-3.375,-3.000616910623168) -- (-3.34375,-3.000777686663762) -- (-3.3125,-3.000923288569151) -- (-3.28125,-3.001048236276683) -- (-3.25,-3.00114441347656) -- (-3.21875,-3.001204202920352) -- (-3.1875,-3.0012197822478903) -- (-3.15625,-3.001183094199434) -- (-3.125,-3.001090411996459) -- (-3.09375,-3.0009378142597503) -- (-3.0625,-3.0007126752673776) -- (-3.03125,-3.0004040746145746) -- (-3.0,-3) -- (-2.96875,-2.9994916286432756) -- (-2.9375,-2.99889280993443) -- (-2.90625,-2.9982217548957633) -- (-2.875,-2.9974876460967668) -- (-2.84375,-2.9967029220796886) -- (-2.8125,-2.9958908146749246) -- (-2.78125,-2.9950725861803) -- (-2.75,-2.99427792324216) -- (-2.71875,-2.993531906295912) -- (-2.6875,-2.992816097662295) -- (-2.65625,-2.992106485532408) -- (-2.625,-2.9913960151946473) -- (-2.59375,-2.990674613659308) -- (-2.5625,-2.9899248867825508) -- (-2.53125,-2.989135872001794) -- (-2.5,-2.9882812) -- (-2.46875,-2.987345741175541) -- (-2.4375,-2.9863891174583688) -- (-2.40625,-2.985479684837437) -- (-2.375,-2.984684533459419) -- (-2.34375,-2.984073342747989) -- (-2.3125,-2.9836711563291245) -- (-2.28125,-2.9835015354768926) -- (-2.25,-2.98359677636716) -- (-2.21875,-2.9839891128744545) -- (-2.1875,-2.984743600258931) -- (-2.15625,-2.985929676324001) -- (-2.125,-2.987581360809295) -- (-2.09375,-2.9897346761265036) -- (-2.0625,-2.9924734809353257) -- (-2.03125,-2.995866620597123) -- (-2.0,-3) -- (-1.96875,-3.0049329218107337) -- (-1.9375,-3.0105511090775847) -- (-1.90625,-3.0167110616254686) -- (-1.875,-3.0233048544128005) -- (-1.84375,-3.0302062746177145) -- (-1.8125,-3.0372913507699932) -- (-1.78125,-3.0444428297228696) -- (-1.75,-3.05149840175776) -- (-1.71875,-3.0583111482282326) -- (-1.6875,-3.064867971580449) -- (-1.65625,-3.0711708037511305) -- (-1.625,-3.077185078588809) -- (-1.59375,-3.082883829399171) -- (-1.5625,-3.08822544288198) -- (-1.53125,-3.093155161118812) -- (-1.5,-3.0976562) -- (-1.46875,-3.1016910785987424) -- (-1.4375,-3.1050886442123584) -- (-1.40625,-3.107665473898055) -- (-1.375,-3.10922853833914) -- (-1.34375,-3.1095846316451694) -- (-1.3125,-3.1086509755777842) -- (-1.28125,-3.106347175558019) -- (-1.25,-3.10258482675776) -- (-1.21875,-3.0972717965095677) -- (-1.1875,-3.0901912788772505) -- (-1.15625,-3.0811123645104384) -- (-1.125,-3.069901419888234) -- (-1.09375,-3.056419344146189) -- (-1.0625,-3.0404198188398883) -- (-1.03125,-3.0216990338718825) -- (-1.0,-3) -- (-0.96875,-2.9751388087857307) -- (-0.9375,-2.947392622861898) -- (-0.90625,-2.917113925317901) -- (-0.875,-2.88458400170536) -- (-0.84375,-2.8501287392540497) -- (-0.8125,-2.814014017802221) -- (-0.78125,-2.7764936053338385) -- (-0.75,-2.73792963652316) -- (-0.71875,-2.6986564193481803) -- (-0.6875,-2.6587859274648995) -- (-0.65625,-2.6184078949810354) -- (-0.625,-2.577651484279625) -- (-0.59375,-2.5366349830563975) -- (-0.5625,-2.495562414356852) -- (-0.53125,-2.454648424414504) -- (-0.5,-2.414062) -- (-0.46875,-2.373983608006993) -- (-0.4375,-2.3346741702860134) -- (-0.40625,-2.2963908782788627) -- (-0.375,-2.25942847183127) -- (-0.34375,-2.224067885761685) -- (-0.3125,-2.1904625582598256) -- (-0.28125,-2.158760885712778) -- (-0.25,-2.1291035521481603) -- (-0.21875,-2.1016370394716386) -- (-0.1875,-2.0767395658382135) -- (-0.15625,-2.0548063898841424) -- (-0.125,-2.036101952934985) -- (-0.09375,-2.0208919770887763) -- (-0.0625,-2.0095720933359287) -- (-0.03125,-2.0024694666178253) -- (0.0,-2) -- (0.03125,-2.0024694666178253) -- (0.0625,-2.0095720933359287) -- (0.09375,-2.020891977088776) -- (0.125,-2.036101952934985) -- (0.15625,-2.0548063898841424) -- (0.1875,-2.0767395658382135) -- (0.21875,-2.1016370394716386) -- (0.25,-2.1291035521481603) -- (0.28125,-2.1587608857127782) -- (0.3125,-2.1904625582598256) -- (0.34375,-2.224067885761685) -- (0.375,-2.2594284718312694) -- (0.40625,-2.2963908782788622) -- (0.4375,-2.3346741702860134) -- (0.46875,-2.373983608006993) -- (0.5,-2.414062) -- (0.53125,-2.454648424414504) -- (0.5625,-2.4955624143568516) -- (0.59375,-2.5366349830563975) -- (0.625,-2.577651484279625) -- (0.65625,-2.6184078949810354) -- (0.6875,-2.6587859274648995) -- (0.71875,-2.6986564193481803) -- (0.75,-2.7379296365231602) -- (0.78125,-2.7764936053338385) -- (0.8125,-2.814014017802221) -- (0.84375,-2.8501287392540497) -- (0.875,-2.88458400170536) -- (0.90625,-2.917113925317901) -- (0.9375,-2.947392622861898) -- (0.96875,-2.9751388087857307) -- (1.0,-3) -- (1.03125,-3.0216990338718825) -- (1.0625,-3.0404198188398883) -- (1.09375,-3.056419344146189) -- (1.125,-3.069901419888234) -- (1.15625,-3.0811123645104384) -- (1.1875,-3.0901912788772505) -- (1.21875,-3.0972717965095677) -- (1.25,-3.10258482675776) -- (1.28125,-3.106347175558019) -- (1.3125,-3.1086509755777842) -- (1.34375,-3.1095846316451694) -- (1.375,-3.10922853833914) -- (1.40625,-3.107665473898055) -- (1.4375,-3.1050886442123584) -- (1.46875,-3.1016910785987424) -- (1.5,-3.0976562) -- (1.53125,-3.093155161118812) -- (1.5625,-3.08822544288198) -- (1.59375,-3.082883829399171) -- (1.625,-3.077185078588809) -- (1.65625,-3.0711708037511305) -- (1.6875,-3.064867971580449) -- (1.71875,-3.0583111482282326) -- (1.75,-3.05149840175776) -- (1.78125,-3.0444428297228696) -- (1.8125,-3.0372913507699932) -- (1.84375,-3.0302062746177145) -- (1.875,-3.0233048544128005) -- (1.90625,-3.0167110616254686) -- (1.9375,-3.0105511090775847) -- (1.96875,-3.0049329218107337) -- (2.0,-3) -- (2.03125,-2.995866620597123) -- (2.0625,-2.9924734809353257) -- (2.09375,-2.9897346761265036) -- (2.125,-2.987581360809295) -- (2.15625,-2.985929676324001) -- (2.1875,-2.984743600258931) -- (2.21875,-2.9839891128744545) -- (2.25,-2.98359677636716) -- (2.28125,-2.9835015354768926) -- (2.3125,-2.9836711563291245) -- (2.34375,-2.984073342747989) -- (2.375,-2.984684533459419) -- (2.40625,-2.985479684837437) -- (2.4375,-2.9863891174583688) -- (2.46875,-2.987345741175541) -- (2.5,-2.9882812) -- (2.53125,-2.989135872001794) -- (2.5625,-2.9899248867825508) -- (2.59375,-2.990674613659308) -- (2.625,-2.9913960151946473) -- (2.65625,-2.992106485532408) -- (2.6875,-2.992816097662295) -- (2.71875,-2.993531906295912) -- (2.75,-2.99427792324216) -- (2.78125,-2.9950725861803) -- (2.8125,-2.9958908146749246) -- (2.84375,-2.9967029220796886) -- (2.875,-2.9974876460967668) -- (2.90625,-2.9982217548957633) -- (2.9375,-2.99889280993443) -- (2.96875,-2.9994916286432756) -- (3.0,-3) -- (3.03125,-3.0004040746145746) -- (3.0625,-3.0007126752673776) -- (3.09375,-3.0009378142597503) -- (3.125,-3.001090411996459) -- (3.15625,-3.001183094199434) -- (3.1875,-3.0012197822478903) -- (3.21875,-3.001204202920352) -- (3.25,-3.00114441347656) -- (3.28125,-3.001048236276683) -- (3.3125,-3.000923288569151) -- (3.34375,-3.000777686663762) -- (3.375,-3.000616910623168) -- (3.40625,-3.000446856397567) -- (3.4375,-3.0002796533496485) -- (3.46875,-3.000126521508723) -- (3.5,-3.0);
    
    \foreach \y in {1,2,3}
    \draw (-0.4,2.2-\y*1.5) node[left] {$\phi^{(\y)}$};
    \end{tikzpicture}
    \caption{Deslaurier-Dubuc interpolets of degree 1, 2 and 3.}
    \label{fig:interpolets}
\end{figure}

\begin{restatable}{lemma}{LemmaRegularityInterpolet}
\label{prop:regularity interpolet L=2}
The interpolet of degree $2$, $\phi^{(2)}$ is $\mathcal C^2$ and has support 
$
    Supp\left(\phi^{(2)}\right) = [-3,3].
$
\end{restatable}
\begin{proof}
See Appendix \ref{appedix sec: proof of lower bound sec}.
\end{proof}
We will now use the interpolate $\phi^{(2)}$ to construct the subspace $X_{N+1}$. Using the sparse grids approach, we can construct a hierarchical basis in $X^{2,\infty}(\Omega)$ using the interpolate  $\phi^{(2)}$ as mother function. In the remaining of the proof, we will write $\phi$ instead of $\phi^{(2)}$ for simplicity, and use the notations and definitions related to sparse grids, introduced in Subsection \ref{subsection: sparse grids}. Because $E_N(K)$ is decreasing in $N$, it suffices to show the result for $N_n$ when we define our space $X_{N+1}$ to be exactly the approximating space $V^{(1)}_n$ of sparse grids $X_{N_n+1} := V^{(1)}_n$. The following  equation establishes the relation between $n$ and $N_n$. In the following, for simplicity, we will write $N$ instead of $N_n$.
\begin{equation*}
    N=dim(V^{(1)}_n)-1 = \sum_{|\bold l|_1\leq n+d-1} 2^{|\mb l|_1-\mb 1}-1 =\sum_{i=0}^{n-1}2^{i-d}\binom{d-1+i}{d-1}-1 = 2^n\cdot  \left(\frac{n^{d-1}}{(d-1)!}+ O(n^{d-2})\right).
\end{equation*}
First, let us give some properties about the subspace $ X_{N+1}$.

\begin{restatable}{proposition}{PropStencilForCoefficients}
\label{prop:stencil for coefficients}
Let $u\in  X_{N+1}$ and write it in decomposed form $u = \sum_{\mb l, \mb i} v_{\mb l,\mb i}\cdot \phi_{\mb l,\mb i}$, where the sum is taken over all multi-indices corresponding to basis functions of $ X_{N+1}$. The coefficients $v_{\mb l,\mb i}$ can be computed in the following way.
\begin{equation*}
    v_{\mb l,\mb i} = \left(\prod_{j=1}^d I_{l_j,i_j}\right)u =: I_{\mb l,\mb i}u
\end{equation*}
where $ I_{l_j,i_j}u = u(\frac{i_j}{2^j})- \frac{9}{16}u (\frac{i_j-1}{2^j})-\frac{9}{16}u (\frac{i_j+1}{2^j}) + \frac{1}{16}u (\frac{i_j-3}{2^j})+\frac{1}{16}u (\frac{i_j+3}{2^j})$. Here, $I_{\mb l,\mb i}$ denotes the $d-$dimensional stencil which gives a linear combination of values of $u$ at $5^d$ nodal points. %We can also represent the coefficients in a integral fashion. Let $\psi_{l_j,i_j}^{(\alpha_j)}(x_j) = \left(1-\frac{1}{\alpha_j}|2^{l_j}x-i_j|\right)_+$ be the wedge centered at $\frac{i_j}{2^{l_j}}$ with width $\frac{\alpha_j}{2^{l_j}}$. Then,
%\begin{equation*}
%    v_{\mb l,\mb i} = (-1)^d2^{-d} 2^{-|\mb l|_1} \int_\Omega  \prod_{j=1}^d \left(\frac{9}{8}\psi_{l_j,i_j}^{(1)}(x_j)-\frac{3}{8}\psi_{l_j,i_j}^{(3)}(x_j)\right)\cdot D^{\mb 2} u(\mb x) d\mb x.
%\end{equation*}
\end{restatable}
\begin{proof}
See Appendix \ref{appedix sec: proof of lower bound sec}.
\end{proof}
Note that the stencil representation of the coefficients gives directly $|v_{\mb l,\mb i}|\leq 5^d \|u\|_\infty$. We are now ready to make our estimates. The goal is to compute $\sup\{\rho: \; \rho U( X_{N+1})\subset K\}$, which will lead to a bound on $b_N(K)_{\mathcal X}$. In order to do that, it suffices to upper bound the Korobov norm by the $L^\infty$ norm for elements of $X_{N+1}$. In fact, if $\Gamma_d>0$ satisfies for all $u\in U(X_{N+1})$,  $|u|_{X^{2,\infty}} \leq \Gamma_d \|u\|_{\infty}$, then $b_N(K)_{\mathcal X} \geq 1/\Gamma_d$. 

Now take $u \in  X_{N+1}$ and let us write $u = \sum_{\mb l, \mb i} v_{\mb l,\mb i}\cdot \phi_{\mb l,\mb i}$. Note that basis functions in the same hierarchical class $W_{\mb l}$ are almost disjoint. More precisely, at each point $\mb x\in \Omega$, at most $3^d$ basis functions $\phi_{\mb l,\mb i}$ have non-zero contribution to $u(\mb x)$. Therefore, for any $\mb 0\leq \mb\alpha\leq 2\cdot \mb 1$,
\begin{align*}
    \|D^{\mb \alpha }u\|_{\infty} &\leq \sum_{|\mb l|_1\leq n+d-1} 3^d \max_{\mb 1\leq \mb i\leq 2^{\mb l}-\mb 1, \; \mb i \text{ odd}} |v_{\mb l,\mb i}| \cdot 2^{\langle \mb\alpha,\mb l\rangle}\|D^{\mb\alpha}\phi\|_\infty \\
    &= 60^d |\phi|_{X^{2,\infty}(\Omega)} \|u\|_\infty\cdot \sum_{i=0}^{n-1} 2^{2i} \binom{d-1+i}{d-1}\\
    &= \frac{60^d}{(d-1)!}|\phi|_{X^{2,\infty}(\Omega)} \cdot 2^{2n}\left(n^{d-1}+O(n^{d-2})\right) \|u\|_\infty.
\end{align*}
Finally, denoting by $C_d$ the constant $\frac{60^d}{ (d-1)!}|\phi|_{X^{2,\infty}(\Omega)}$, we get for $n$ sufficiently large
\begin{equation*}
    b_N(K)_{\mathcal X}\geq \frac{1}{2C_d}\cdot
    \frac{1}{2^{2n}\cdot n^{d-1}}.
\end{equation*}
Furthermore, recall $N = \frac{1}{(d-1)!}2^n\cdot n^{d-1}\cdot\left(1+O\left(\frac{1}{n}\right)\right).$ Therefore, $n\sim\frac{\log N}{\log 2}$, and $2^n\sim (d-1)!(\log 2)^{d-1}\cdot\frac{N}{(\log N)^{d-1}}.$ Finally we obtain for some constant $c_d>0$,
\begin{equation*}
    b_{N}(K)_{\mathcal X} \geq c_d \frac{1}{N^2}(\log N)^{d-1}.
\end{equation*}
We conclude by analyzing the minimum number of parameters in order to get an $\epsilon-$approximation of the Korobov unit ball $K$. Define $n_\epsilon:=\min\{n: \; \frac{2}{C_d} \frac{1}{2^{2n}\cdot n^{d-1}} \leq \epsilon\}$. This yields $n_\epsilon \sim \frac{1}{2\log 2}\log\frac{1}{\epsilon}$, and $2^n\sim \sqrt{\frac{2}{C_d}}(2\log 2)^{d-1}\frac{1}{\sqrt \epsilon \cdot(\log\frac{1}{\epsilon})^{\frac{d-1}{2}}}$. The number of needed parameters to obtain an $\epsilon$ approximation $K$ is therefore
\begin{equation*}
    N_\epsilon \sim \tilde C_d \cdot \frac{1}{\sqrt\epsilon }\left(\log \frac{1}{\epsilon}\right)^{\frac{d-1}{2}},
\end{equation*}
for some constant $\tilde C_d>0$.
\end{proof}

\paragraph{Approximating separable functions.}
We conclude this section by proving another optimality result for neural networks' approximation power, for the set of separable functions with regularity $r$.

\begin{equation*}
    \mathcal{S}^r = 
    \{ f \in W^{r,\infty}(\Omega)
    \; : \; 
    \exists g_1,\ldots, g_d \in \mathbb{R}^{\Omega}, \; \forall \bold{x}\in \Omega, \;
    f(\bold{x}) = \sum_{i=1}^d g_i(x_i)
    \}
\end{equation*}
This set of functions is motivated by Kolmogorov–Arnold representation theorem \cite{kolmogorov1957representation} that states that every multivariate continuous function can be represented as a superposition of continuous functions of one variable. However the functions involved in the construction are rather pathological. To approximate separable functions, it suffices to approximate each of the $d$ functions composing it. Using Mhaskar's result, Theorem \ref{thm: Mhaskar upper bound}, we can approximate each of these functions within $\epsilon/d$ by a neural network with $O(d^{1/r}\epsilon^{-1/r}) = O(\epsilon^{-1/r})$ parameters. Combining these approximations, we need  $O(d\epsilon^{-1/r}) = O(\epsilon^{-1/r})$ parameters to approximate the space separable functions within $\epsilon$. We prove a lower bound on the minimal number of parameters needed to approximate the space of separable functions by continuous function approximators, matching neural networks' upper bound, and therefore, proving their optimality.

\begin{proposition}
Let $r\geq 0$. Take $\mathcal{X} = L^{\infty}([0,1]^d)$ and $K = \{f \in \mathcal{S}^{r} \; : \; |f|_{W^{r,\infty}}\leq 1 \}$ the unit ball of separable functions. Then, there exists a universal constant $c>0$ such that
\begin{equation*}
E_{N}(K) \geq \frac{cd^{r-1}}{(N+1)^r}.
\end{equation*}
Equivalently, for all $\epsilon>0$, a learning algorithm with continuous parameters approximating the unit ball of separable functions within $\epsilon$ in $L^{\infty}$ norm uses at least $\Theta(\epsilon^{-\frac{1}{r}})$ parameters.
\end{proposition}
\begin{proof}
As $E_N(K)$ is decreasing in $N$, it suffices to show the result for $N$ verifying $N+1=dm$ for $m \in \mathbb{N}$. Let $N+1=dm$ with $m \in \mathbb{N}$. Let $\phi \in C^{\infty}(\mathbb{R})$ such that $\phi$ is $1$ in $[\frac{1}{4},\frac{3}{4}]$ and $0$ in $\mathbb{R} \setminus [0,1]$. For $i=0, \ldots,m-1$ and $j=1, \ldots, d$, consider the function of support $\left[\frac{i+1/4}{m},\frac{i+3/4}{m}\right]$,
\begin{equation*}
    \phi_{i,j}(\bold{x}) = \phi(mx_j - i), \; \forall \bold{x} \in [0,1]^d.
\end{equation*}
We use DeVore's result, Theorem \ref{theorem: Bernstein lower bound}, with $X_{N+1} = \text{Span}\{ \phi_{i,j}\; : \; i=0, \ldots,m-1, \; j=1, \ldots, d\}$. \comment{Let $C_0>0 = \max_{\mb\alpha \; : \; |\mb\alpha|\leq r } \|D^{\mb\alpha}\phi\|_{\infty}$. }Notice that for all $\mb\alpha$ such that $|\mb\alpha|\leq r$, 
$
\|D^{\mb\alpha}\phi_{i,j}\|_{\infty} \leq m^{|\mb\alpha|}\cdot |\phi|_{W^{r,\infty}}
$.
Let $g \in U(X_{N+1})$ where $U(X_{N+1})$ is the unit ball with respect to the Sobolev norm of order $r$. $g$ can be written as $g = \sum_{i=1}^m\sum_{j=1}^d c_{i,j}\phi_{i,j}$. Let $\mb\alpha$ a multi-integer such that $|\mb\alpha| \leq r$. Notice that for each $j$, only one of $\phi_{i,j}$, $i=1,\ldots,m$, is involved in the value of $D^{\mb\alpha}g$. Therefore,
\comment{\begin{align*}
    \| D^{\mb\alpha}g \|_{\infty} 
    &\leq \max_{i_1,\ldots, i_d} \sum_{j=1}^d|c_{i_j,j}|\cdot\|D^{\mb\alpha} \phi_{i_jj}\|_{\infty} \\
    & \leq |\phi|_{W^{r,\infty}}\cdot m^{|\mb\alpha|} \max_{i_1,\ldots, i_d} \sum_{j=1}^d|c_{i_jj}| \\
    & \leq |\phi|_{W^{r,\infty}}\cdot d m^{|\mb\alpha|}\max_{i,j}|c_{i,j}|
\end{align*}}
\begin{equation*}
    \| D^{\mb\alpha}g \|_{\infty} 
    \leq \max_{i_1,\ldots, i_d} \sum_{j=1}^d|c_{i_j,j}|\cdot\|D^{\mb\alpha} \phi_{i_jj}\|_{\infty} 
    \leq d m^{|\mb\alpha|}|\phi|_{W^{r,\infty}}\cdot \max_{i,j}|c_{i,j}|
\end{equation*}
For all $i = 1, \ldots, m$ and $j=1,\ldots,d$, by choosing $x_j = \frac{i_j +1/2}{m}$ and $x_l = \frac{1}{8m}$ for  $l \neq j$, we have $|g(\bold{x})| = |c_{ij}|$ which implies $\|g\|_{\infty} \geq |c_{ij}|$. Hence, 
\begin{align*}
    \| D^{\mb\alpha}g \|_{\infty} 
    &\leq d m^{|\mb\alpha|}|\phi|_{W^{r,\infty}}\cdot  \|g\|_{\infty}
\end{align*}
This implies that $|g|_{W^{r,\infty}} \leq  2d^{-r+1} |\phi|_{W^{r,\infty}}\cdot (N+1)^{r}\|g\|_{\infty}$. Notice that the functions of the space $X_{N+1}$ are separable as each $\phi_{i,j}$ is applied to only one coordinate, therefore, $\frac{d^{r-1}}{|\phi|_{W^{r,\infty}}\cdot (N+1)^r}U(X_{N+1}) \subset K$. Using Theorem \ref{theorem: Bernstein lower bound} we have the desired result for $N$.
\end{proof}

\section{Conclusion and Discussion}
We have proved new upper and lower bounds on the number of neurons and training parameters needed by shallow and deep neural networks to approximate functions of the Korobov space. Our bounds show that shallow and deep networks, not only break the curse of dimensionality but are also near-optimal function approximators of space of Korobov functions. This work, contributes therefore to understanding the practical success of neural networks theoretically.

Our work suggests several extensions. First, we have constructed throughout the paper explicitly deep and shallow structures nearly reaching the optimal number of parameters needed to approximate Korobov functions. It would be very interesting to see if the theoretical near-optimality of these structures translates into powerful empirical performance. While commonly used structures (e.g. Convolution Neural Networks, or Recurrent Neural Networks) are motivated by properties of the data such as symmetries, our structures are motivated by theoretical insights on how to optimally approximate a large general class of functions with a given number of neurons and parameters. These new structures might prove to be competitive with commonly used structures and constitutes new building blocks of neural networks in learning tasks.

Second, our upper bounds (Theorem \ref{thm: shallow NN bound.} and \ref{thm: fixed deep upper bound}) nearly match our lower bound (Theorem \ref{thm:lower bound korobov}) on the least number of training parameters needed to approximate functions of the Korobov space. We wonder if it is possible to close the gap between these bounds and prove neural network's optimality as function approximators. Using our approach, proving neural network's optimality would boil down to proving that sparse grids are optimal function approximators by improving our lower bound to match sparse grid number of parameters $O(\epsilon^{-\frac{1}{2}}(\log \frac{1}{\epsilon})^{\frac{3(d-1)}{2}})$.

Finally, we showed the near-optimality of neural networks among the set of continuous function approximators. An interesting question would be to explore lower bounds (analog to Theorem \ref{thm:lower bound korobov}) that we can get when considering larger sets of function approximators. One possibility would be that the curse of dimensionality vanishes for the Sobolev space when considering discontinuous function approximators. In this case, it would be insightful to investigate whether some discontinuous neural network construction can break the curse of dimensionality for the Sobolev space leading to stronger approximation results. The question is then whether neural networks are still near-optimal in these larger sets of function approximators.

\section*{Acknowledgments}
The authors are grateful to Tomaso Poggio and the MIT 6.520 course teaching staff for several discussions, remarks and comments that were useful to this work.

\bibliography{References}{}
\bibliographystyle{plain}

\newpage

\section*{Appendix}
\appendix

\section{Proofs of Section 3}
\label{appendix sec: NN upper bound sec}
\subsection{Proof of Lemma \ref{lemma:1D function approx.}}
\LemmaOneDFunctionApprox*

\begin{proof}
Let $m=\left\lfloor\frac{d-c}{\epsilon}\right\rfloor$ Define a subdivision of the image interval $c\leq y_1\leq \cdots \leq y_m\leq d$ where $y_k = c+ k\epsilon$ for $k=1,\cdots,m$. Note that this subdivision contains exactly $\left\lceil\frac{d-c}{\epsilon}\right\rceil$ pieces. Now define a subdivision of $I$, $x_1\leq x_2\leq \cdots \leq x_m$ by
\begin{equation*}
    x_k := \sup\{x\in I, f(x)\leq y_k\},
\end{equation*}
for $k=1,\cdots,m$. This subdivision stills has $\left\lceil\frac{d-c}{\epsilon}\right\rceil$ pieces. We now construct our approximation function $\hat f$ on $I$ as the continuous piece-wise affine function on the subdivision $x_1\leq \cdots \leq x_m$ such that $\hat f(x_k) = y_k$ for all $1\leq k\leq m$ and $\hat f$ is constant before $x_1$ and after $x_m$ (see Figure \ref{fig:piecewise approx}). Let $x\in I$.
\begin{itemize}
    \item If $x\leq x_1$, because $f$ is increasing and right-continuous, $c\leq f(x)\leq f(x_1)\leq y_1=c+\epsilon$. Therefore $|f(x)-\hat f(x)| = |f(x)-(c+\epsilon)| \leq \epsilon$.
    \item If $x_k < x\leq x_{k+1}$, we have $y_k<f(x)\leq f(x_{k+1})\leq y_{k+1}$. Further note that $y_k \leq \hat f(x)\leq y_{k+1}$. Therefore $|f(x)-\hat f(x)| \leq y_{k+1}-y_k = \epsilon$.
    \item If $x_m<x$, then $y_m<f(x)\leq d$. Again, $|f(x)-\hat f(x)| = |f(x)-y_m| \leq d-y_m\leq \epsilon$.
\end{itemize}
Therefore $\|f-\hat f\|_{\infty} \leq \epsilon$. We can now use Lemma \ref{lemma:piece-wise affine} to end the proof.
\end{proof}

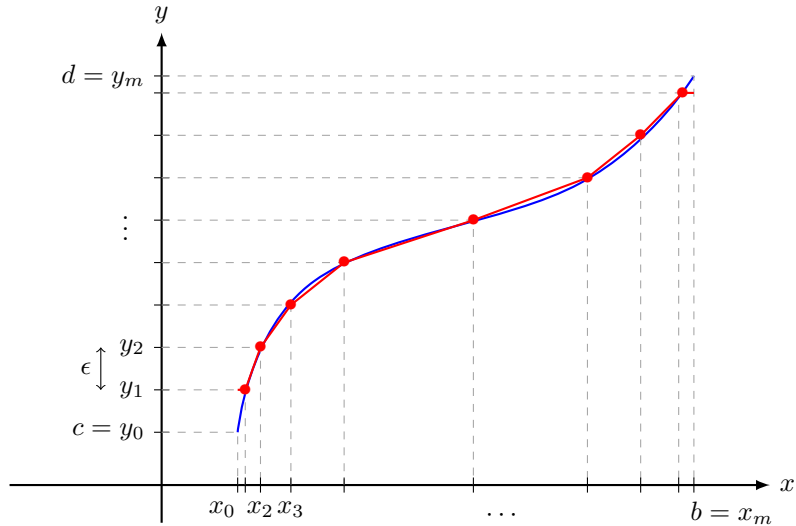
\begin{figure}[h!]
    \centering
    \begin{tikzpicture}
    %axis
    \draw [thick,->,>=latex] (-2,0) -- (8,0) node[right] {$x$};
    \draw [thick,->,>=latex] (0,-0.5) -- (0,6) node[above] {$y$};
    \draw[thick,blue] plot[smooth, tension=.7] coordinates {(1,0.7) (2,2.7) (5.5,4) (7,0.7 + 5.2*1.05-0.7*1.05)};
    \draw (-0.1,0.7) node[left] {$c=y_0$} -- (0.1,0.7);
    \foreach \k in {1,2}{
        \draw (-0.1,0.7 + 5.2*\k/8-0.7*\k/8) node[left] {$y_{\k}$} -- (0.1,0.7 + 5.2*\k/8-0.7*\k/8);
    }
    \draw (-0.3,0.7 + 5.2*5/8-0.7*5/8) node[left] {$\vdots$};
    
    \foreach \k in {3,4,5,6,7,8}{
        \draw (-0.1,0.7 + 5.2*\k/8-0.7*\k/8) -- (0.1,0.7 + 5.2*\k/8-0.7*\k/8);
    }
    \draw (-0.1,0.7 + 5.2*1.05-0.7*1.05) node[left] {$d=y_m$} -- (0.1,0.7 + 5.2*1.05-0.7*1.05);
    
    \draw [<->] (-0.8,0.7 + 5.2/4-0.7/4) -- node[left] {$\epsilon$} (-0.8,0.7 + 5.2/8-0.7/8);
    
    \draw [dashed, gray!70] (0,0.7) -- (1,0.7) -- (1,0);
    \draw [dashed, gray!70] (0,0.7 + 4.5*1/8) -- (1.1,0.7 + 4.5*1/8) -- (1.1,0);
    \draw [dashed, gray!70] (0,0.7 + 4.5*2/8) -- (1.3,0.7 + 4.5*2/8) -- (1.3,0);
    \draw [dashed, gray!70] (0,0.7 + 4.5*3/8) -- (1.7,0.7 + 4.5*3/8) -- (1.7,0);
    \draw [dashed, gray!70] (0,0.7 + 4.5*4/8) -- (2.4,0.7 + 4.5*4/8) -- (2.4,0);
    \draw [dashed, gray!70] (0,0.7 + 4.5*5/8) -- (4.1,0.7 + 4.5*5/8) -- (4.1,0);
    \draw [dashed, gray!70] (0,0.7 + 4.5*6/8) -- (5.6,0.7 + 4.5*6/8) -- (5.6,0);
    \draw [dashed, gray!70] (0,0.7 + 4.5*7/8) -- (6.3,0.7 + 4.5*7/8) -- (6.3,0);
    \draw [dashed, gray!70] (0,0.7 + 4.5*8/8) -- (6.8,0.7 + 4.5*8/8) -- (6.8,0);
    \draw [dashed, gray!70] (0,0.7 + 4.5*1.05) -- (7,0.7 +4.5*1.05)  -- (7,0);
    
    \draw  (1,0.1)  -- (1,-0.1);
    \draw  (1.1,0.1)  -- (1.1,-0.1);
    \draw  (1.3,0.1) -- (1.3,-0.1)  node[below] {$x_2$};
    \draw  (1.7,0.1) -- (1.7,-0.1) node[below] {$x_3$};
    \draw  (2.4,0.1) -- (2.4,-0.1);
    \draw  (4.1,0.1) -- (4.1,-0.1);
    \draw  (5.6,0.1) -- (5.6,-0.1);
    \draw  (6.3,0.1) -- (6.3,-0.1);
    \draw  (6.8,0.1) -- (6.8,-0.1);
    \draw  (7,0.1) -- (7,-0.1);
    
    \draw [thick,red] (1,0.7 + 4.5*1/8) -- (1.1,0.7 + 4.5*1/8) -- (1.3,0.7 + 4.5*2/8) -- (1.7,0.7 + 4.5*3/8) -- (2.4,0.7 + 4.5*4/8) -- (4.1,0.7 + 4.5*5/8) -- (5.6,0.7 + 4.5*6/8) -- (6.3,0.7 + 4.5*7/8) -- (6.85,0.7 + 4.5*8/8) -- (7,0.7 + 4.5*8/8);
    
    \draw (4.5,-0.4) node {$\cdots$};
    
    \draw  [red] (1.1,0.7 + 4.5*1/8) node {$\bullet$};
    \draw  [red] (1.3,0.7 + 4.5*2/8) node {$\bullet$};
    \draw  [red] (1.7,0.7 + 4.5*3/8) node {$\bullet$};
    \draw  [red] (2.4,0.7 + 4.5*4/8) node {$\bullet$};
    \draw  [red] (4.1,0.7 + 4.5*5/8) node {$\bullet$};
    \draw  [red] (5.6,0.7 + 4.5*6/8) node {$\bullet$};
    \draw  [red] (6.3,0.7 + 4.5*7/8) node {$\bullet$};
    \draw  [red] (6.85,0.7 + 4.5*8/8) node {$\bullet$};
    
    \draw (0.8,-0.1) node[below] {$x_0$};
    \draw (7.5,-0.1) node[below] {$b=x_m$};
    %\draw [blue] (4.5,3) node {$f$};
    %\draw [red] (4.2,4.2) node {$\hat f$};
    \end{tikzpicture}
    \caption{Approximation of a right-continuous increasing function (blue) in an interval $[c,d]$ within $\epsilon$ by a piece-wise linear function (red) with $\lfloor \frac{d-c}{\epsilon} \rfloor$ pieces. The approximation is constructed using a regular subdivision of the $y$ axis of step $\epsilon$ and constructing a linear approximation in the pre-image of each part of the subdivision.}
    \label{fig:piecewise approx}
\end{figure}
\subsection{Proof of Lemma \ref{lemma:quadratic approximation 1D}}

\LemmaQuadraticApproximation*

\begin{proof}
Similarly as the proof of Lemma \ref{lemma:1D function approx.}, the goal is to approximate $f$ by a piece-wise affine function $\hat f$ defined on a subdivision $x_0=a \leq x_1\leq \cdots \leq x_m\leq x_{m+1} = b$ such that $f$ and $\hat f$ coincide on $x_0,\cdots,x_{m+1}$. We first analyse the error induced by a linear approximation of the function on each piece. Let $x\in [u,v]$ for $u,v\in I$. Using the mean value theorem, there exists $\alpha_x \in [u,x]$ such that $f(x)-f(u) = f'(\alpha_x) (x-u)$ and $\beta_x \in [x,v]$ such that $f(v)-f(x) = f'(\beta_x) (v-x)$. Combining these two equalities, we get,
\begin{align*}
f(x)-f(u) - (x-u)\frac{f(v)-f(u)}{v-u}
&= 
\frac{(v-x)(f(x)-f(u)) - (x-u)(f(v)-f(x))}{v-u}\\
&= (x-u)(x-v)\frac{f'(\beta_x)-f'(\alpha_x)}{v-u}\\
&= (x-u)(v-x)\frac{\int_{\alpha_x}^{\beta_x}f''(t)dt}{v-u}
\end{align*}

\noindent Hence, 
\begin{equation}
\label{eq:taylor_like expansion}
    f(x) = f(u) + (x-u) \frac{f(v)-f(u)}{v-u} + (x-u)(v-x)\frac{\int_{\alpha_x}^{\beta_x}f''(t)dt}{v-u}.
\end{equation}
% \begin{align*}
%     (x-u)(x-v)\frac{f'(\beta_x)-f'(\alpha_x)}{v-u} &= \frac{(v-x)(f(x)-f(u)) - (x-u)(f(v)-f(x))}{v-u}\\
%     &= f(x)-f(u) - (x-u)\frac{f(v)-f(u)}{v-u}.
% \end{align*}
We now apply this result to bound the approximation error on each pieces of the subdivision. Let $k\in [m]$. Recall $\hat{f}$ is linear on the subdivision $[x_k,x_{k+1}]$ and $\hat{f}(x_k)= f(x_k)$ and $\hat{f}(x_{k+1})= f(x_{k+1})$. Hence, for all $x\in [x_k,x_{k+1}], \; \hat{f}(x)= f(x_k) + (x-x_k)\frac{f(x_{k+1})-f(x_{k})}{x_{k+1}-x_k}$. Using Equation \eqref{eq:taylor_like expansion} with $u=x_k$ and $v= x_{k+1}$, we get,
\begin{align*}
    \|f-\hat f\|_{\infty, [x_k,x_{k+1}]} 
    &\leq 
     \sup_{x\in [x_k,x_{k+1}]} 
     \left|  (x-x_k)(x_{k+1}-x)\frac{\int_{\alpha_x}^{\beta_x}f''(t)dt}{x_{k+1}-x_k}   \right|  \\
    &\leq \frac{1}{2} (x_{k+1}-x_k) \int_{x_k}^{x_{k+1}} |f''(t)|dt\\
    &\leq \frac{1}{2} (x_{k+1}-x_k)^2 \|f''\|_{\infty, [x_k,x_{k+1}]}.
\end{align*}
Therefore, using a regular subdivision with step $\sqrt{\frac{2\epsilon}{\|f''\|_{\infty}}}$ yields an $\epsilon-$approximation of $f$ with $\left\lceil \frac{(b-a) \sqrt{\|f''\|_{\infty}}}{\sqrt{2\epsilon}} \right \rceil$ pieces.

We now show that for any $\mu>0$, there exists an $\epsilon-$approximation of $f$ with at most $\frac{\int \sqrt{ |f''|}}{\sqrt {2\epsilon}}(1+\mu)$ pieces. To do so, we use the fact that the upper Riemann sum for $\sqrt{f''}$ converges to the integral since $\sqrt{f''}$ is continuous on $[a,b]$.
First define a partition $a=X_0\leq X_K=b$ of $[a,b]$ such that the upper Riemann sum $\mathcal R(\sqrt{f''})$ on this subdivision satisfies $\mathcal R(\sqrt{f''}) \leq (1+\mu/2)\int_a^b \sqrt{f''}$.
Now define on each interval $I_k$ of the partition a regular subdivision with step $\sqrt{\frac{2\epsilon}{\|f''\|_{I_k}}}$ as before.
Finally, consider the subdivision union of all these subdivisions, and construct the approximation $\hat f$ on this final subdivision. By construction, $\|f-\hat f\|_\infty \leq \epsilon$ because the inequality holds on each piece of the subdivision. Further, the number of pieces is
\begin{equation*}
    \sum_{i=0}^{K-1} 1+ \frac{(X_{i+1}-X_i) \sup_{[X_i,X_{i+1}]}\sqrt{f''}}{\sqrt{2\epsilon}} = \frac{\mathcal R(\sqrt{f''})}{\sqrt{2\epsilon}} +K\leq \frac{\int \sqrt{ |f''|}}{\sqrt {2\epsilon}}(1+\mu),
\end{equation*}
for $\epsilon>0$ small enough. Using  Lemma \ref{lemma:piece-wise affine} we can complete the proof.
\end{proof}

\subsection{Proof of Corollary \ref{corollary:quadratic approximation log}}
\CorQuadraticApproximationLog*
\begin{proof}
In view of Lemma \ref{lemma:quadratic approximation 1D}, the goal is to show that we can remove the dependence of $\mu(f,\epsilon)$ in $\delta$. This essentially comes from the fact that the upper Riemann sum behaves well for approximating $\log$. Consider the subdivision $x_0:=\delta\leq x_1 \leq \cdots \leq x_m \leq x_{m+1}:=1$ with $m=\left \lfloor\frac{1}{\tilde \epsilon}\log\frac{1}{\delta}\right\rfloor$ where $\tilde \epsilon:= \log(1+\sqrt{2\epsilon})$, such that $ x_k = e^{\log \delta + k\tilde \epsilon}$, for $k=0,\cdots,m-1$. Denote $\hat f$ the corresponding piece-wise affine approximation. Similarly to the proof of Lemma \ref{lemma:quadratic approximation 1D}, for $k=0,\cdots,m-1$,
\begin{equation*}
    \|\log-\hat f\|_{\infty, [x_k,x_{k+1}]} \leq \frac{1}{2} (x_{k+1}-x_k)^2 \|f''\|_{\infty, [x_k,x_{k+1}]}\leq \frac{(e^{\tilde \epsilon}-1)^2}{2}\leq \epsilon.
\end{equation*}
The proof follows.
\end{proof}

\subsection{Proof of Lemma \ref{lemma:approx activations}}

\LemmaApproxActivations*

\begin{proof}
Let $\delta,\epsilon>0$ and $\sigma$ a sigmoid-like activation with limit $a$ (resp. $b$) in $-\infty$ (resp. $+\infty$). There exists $x_0>0$ sufficiently large such that $(b-a)|\sigma(x)-a|\leq \epsilon$ for $x\leq -x_0$ and $(b-a)|\sigma(x)-b|\leq \epsilon$ for $x\geq x_0.$ It now suffices to take $M:=x_0/\delta$ to obtain the desired result.

Now let $\sigma$ be a ReLU-like activation with oblique asymptote $bx$ in $+\infty$ where $b>0$. Let $M$ such that $|\sigma|\leq Mb\epsilon$ for $x\leq 0$ and $|\sigma(x)-bx|\leq Mb\epsilon$ for $x\geq 0$. One can check that $|\frac{\sigma(Mx)}{Mb}|\leq \epsilon$ for $x\leq 0$, and $|\frac{\sigma(Mx)}{Mb}-x|\leq \epsilon$ for $x\geq 0$.
\end{proof}

\subsection{Proof of Lemma \ref{lemma:1D function approx sigmoid}}

\begin{lemma}
\label{lemma:1D function approx sigmoid}
Let $\sigma$ be a sigmoid-like activation. Let $f:I \longrightarrow [c,d]$ be a right-continuous increasing function where $I$ is an interval, and let $\epsilon>0$. There exists a shallow neural network with activation $\sigma$, with at most $2\frac{d-c}{\epsilon}$ neurons on a single layer, that approximates $f$ within $\epsilon$ for the infinity norm.
\end{lemma}

\begin{proof}
The proof is analog to that of Lemma \ref{lemma:1D function approx.}. Let $m=\lfloor\frac{d-c}{\epsilon}\rfloor$. We define a regular subdivision of the image interval $c\leq y_1\leq \cdots \leq y_m\leq d $ where $y_k=c+k\epsilon$ for $k=1,\cdots,m$, then using the monotony of $f$, we can define a subdivision of $I$, $x_1\leq\cdots\leq x_m$ such that $x_k :=\sup\{x\in I,f(x)\leq y_k\}$. Let us first construct an approximation neural network $\hat f$ with the Heaviside activation. Consider
\begin{equation*}
    \hat f(x) := y_1 + \epsilon\sum_{i=1}^{m-1} \mb 1\left(x-\frac{x_i+x_{i+1}}{2}\geq 0\right).
\end{equation*}
Let $x\in [c,d]$ and $k$ such that $x \in [x_k,x_{k+1}]$. We have by monotony $y_k \leq f(x) \leq y_{k+1}$ and 
$ y_k = y_1 + (k-1)\epsilon \leq \hat{f}(x) \leq y_1 + k\epsilon = y_{k+1}.$
Hence, $\hat f$ approximates $f$ within $\epsilon$ in infinity norm.
% Note that this results only from the fact that on each interval $[x_i,x_{i+1}]$, the activation function $\mb 1\left(x-\frac{x_i+x_{i+1}}{2}\geq 0\right)$ has values in $[0,1]$.

Let $\delta<\min_{i =1,\ldots,m}(x_{i+1}-x_i)/4$ and $\sigma$ a general sigmoid-like activation with limits $a$ in $-\infty$ and $b$ in $+\infty$. Take $M$ given by Lemma \ref{lemma:1D function approx sigmoid} such that $\frac{\sigma(Mx)}{b-a}-a$ approximates the Heaviside function within $1/m$ outside of $(-\delta,\delta)$ and has values in $[0,1]$. Using the same arguments as above, the function
\begin{equation*}
    \hat f(x) := y_1 + \epsilon\sum_{i=1}^{m-1}  \frac{\sigma\left( Mx-M\frac{x_i+x_{i+1}}{2}\right)}{b-a}-a
\end{equation*}
approximates $f$ within $2\epsilon$ for the infinity norm. The proof follows.
\end{proof}
 
\section{Proofs of Section 5}\label{appedix sec: proof of lower bound sec}
\subsection{Proof of Lemma \ref{prop:regularity interpolet L=2}}

\LemmaRegularityInterpolet*

\begin{proof}

To analyze the regularity of the interpolet, we introduce the trigonometric polynomial
\begin{equation*}
    P^{(2)}(\theta) := \sum_{k\in \Z} \phi^{(2)}\left(\frac{k}{2}\right) e^{ik\theta} = 1+\frac{9}{16}(e^{i\theta}+e^{-i\theta}) - \frac{1}{16} (e^{3i\theta}+e^{-3i\theta}).
\end{equation*}
We can write $P^{(2)}(\theta)$ as
\begin{equation*}
  P^{(2)}(\theta) = \left[\frac{\sin \theta}{\sin\left(\frac{\theta}{2}\right)}\right]^4\cdot S(\theta),
\end{equation*}
where $S(\theta) = \frac{1}{4}-\frac{1}{16}(e^{i\theta}+e^{-i\theta})$ is a trigonometric polynomial of degree $1$. Deslaurier and Dubuc \cite[Theorem 7.11]{deslauriers1989symmetric} showed that the interpolet $\phi^{(2)}$ has regularity $\left\lfloor -\frac{\log r}{\log 2}\right\rfloor$ where $r$ is the spectral radius of the matrix $B:=[s_{j-2k}]_{-1\leq j,k\leq 1}$ where $s_j$ are the coefficients of the trigonometric polynomial $S(\theta)$. In our case, matrix $B$ writes
\begin{equation*}
B = 
\begin{pmatrix}
-1/16 & -1/16 & 0 \\
0 & 1/4 & 0 \\
0 & -1/16 & -1/16
\end{pmatrix},
\end{equation*}
and has spectral radius $r=1/4$. Therefore, the regularity of $\phi^{(2)}$ is at least $\left\lfloor -\frac{\log r}{\log 2}\right\rfloor=2$.
For the support, one can check that more generally $Supp(\phi^{(L)}) = [-2L+1,2L-1].$
\end{proof}

\subsection{Proof of Proposition \ref{prop:stencil for coefficients}}

\PropStencilForCoefficients*
\begin{proof}
We start by looking at the $1-$dimensional case. One can check that $ I_{l ,i } \phi_{\tilde l ,\tilde i } = \mb 1_{\tilde l  = l , \tilde i =i }$. Indeed, if $\tilde l >l $ of if $\tilde l =l $ and $\tilde i \neq i $, then $\phi_{\tilde l ,\tilde i }$ will be zero at the nodal values of $I_{l ,i }$. Further, if $\tilde l <l $ the iterative construction of $\phi$ gives directly.
\begin{equation*}
    \phi_{\tilde l ,\tilde i } \left(\frac{i }{2^j}\right)=  \frac{9}{16}\phi_{\tilde l ,\tilde i } \left(\frac{i +1}{2^j}\right)+\frac{9}{16} \phi_{\tilde l ,\tilde i } \left(\frac{i -1}{2^j}\right) - \frac{1}{16} \phi_{\tilde l ,\tilde i } \left(\frac{i +3}{2^j}\right)-\frac{1}{16} \phi_{\tilde l ,\tilde i } \left(\frac{i -3}{2^j}\right).
\end{equation*}
Therefore, we obtain $I_{l ,i }\phi_{\tilde l ,\tilde i }=0$. Finally, $I_{l ,i }u = \sum_{\tilde l ,\tilde i } v_{\tilde l .\tilde i }\cdot I_{l ,i }\phi_{\tilde l ,\tilde i } = v_{l ,i }\cdot I_{l ,i }\phi_{ l , i } = v_{l ,i }.$
This proves the stencil representation of $v_{\mb l,\mb i}$ for dimension $1$. Finally, using the tensor product approach of the stencil operator $I_{\mb l,\mb i}$ we can generalize the formula to general dimensions.
%Now let us prove the integral formula. Again, we will use the tensor product approach to restrict ourselves to the $1-$dimensional case.
%\begin{align*}
%    I_{l ,i }u &= \frac{9}{8}\left( u \left(\frac{2k+1}{2^{j+1}}\right)-\frac{1}{2}u \left(\frac{k}{2^j}\right)+\frac{1}{2}u \left(\frac{k+1}{2^j}\right) \right) - \frac{1}{8}\left( u \left(\frac{2k+1}{2^{j+1}}\right)-\frac{1}{2}u \left(\frac{k-2}{2^j}\right)+\frac{1}{2}u \left(\frac{k+3}{2^j}\right) \right)\\
 %   &=- \frac{9}{8} 2^{-l -1}\int_\Omega \psi_{l ,i }^{(1)} \cdot \frac{\partial^2 u(x)}{\partial x^2}dx + \frac{3}{8} 2^{-l -1}\int_\Omega \psi_{l ,i }^{(3)}\cdot \frac{\partial^2 u(x)}{\partial x^2}dx\\
  %  &= -2^{-l -1} \int_\Omega \left(\frac{9}{8}\psi_{l ,i }^{(1)} -\frac{3}{8}\psi_{l ,i }^{(3)} \right)\cdot \frac{\partial^2 u(x)}{\partial x^2}dx.
%\end{align*}
%The proof follows from the tensor product construction.
\end{proof}

\end{document}